\documentclass{amsart}

\usepackage{url}

\usepackage{amsmath,amsfonts,amsthm}       
\usepackage{algorithmicx}
\usepackage{algorithm}
\usepackage{algpseudocode}
\usepackage{array}
\usepackage{subcaption}

\usepackage{tikz}
\usetikzlibrary{fit}
\tikzstyle{box} = [rectangle,draw=black,
minimum width=15em,minimum height=4ex]
\tikzstyle{outerbox} = [rectangle,draw=black]
\tikzstyle{arrow} = [->]

\usepackage{multirow}

\newcommand{\bR}{\mathbb{R}}
\newcommand{\cH}{\mathcal{H}}

\newcommand{\cW}{\mathcal{W}}
\newcommand{\cC}{\mathcal{C}}
\newcommand{\cF}{\mathcal{F}}
\newcommand{\dispersion}{m.d. weighting}

\newcommand{\magn}{\mathsf{Mag}}
\newcommand{\magnp}{\mathsf{Mag}_{+}}
\newcommand{\wg}{\mathsf{wg}}
\newcommand{\wgp}{\mathsf{wg}_{+}}

\newcommand{\highlight}[1]{{#1}}

\newtheorem{defn}{Definition}
\newtheorem{lem}[defn]{Lemma}
\newtheorem{prop}[defn]{Proposition}

\newtheorem{cor}[defn]{Corollary}
\newtheorem{ex}[defn]{Example}
\newtheorem{rmk}[defn]{Remark}

\begin{document}

\title{Maximum Diversity and Weighting for invariants of periodic time series}
\author[Byungchang So]{Byungchang So*}
\thanks{*Seoul National University Department of Mathematical Sciences, Seoul, Seoul, South Korea; \url{sinwall@snu.ac.kr}; Corresponding author.\newline\indent
A preprint version of this article was published in \cite{PREPRINT}.
}

\date{}

\begin{abstract}
Magnitude, obtained as a special case of Euler characteristic of enriched category, represents a sense of the size of metric spaces and is related to classical notions such as cardinality, dimension, and volume. 
While the studies have explained the meaning of magnitude from various perspectives, continuity also gives a valuable view of magnitude. 
Based on established results about continuity of magnitude and maximum diversity, this article focuses on continuity of weighting, a distribution whose totality is magnitude, and its variation corresponding to maximum diversity. 
Meanwhile, recent studies also illuminated the connection between magnitude and data analysis by applying magnitude theory to point clouds representing the data or the set of model parameters. 
This article will also provide an application for time series analysis by introducing a new kind of invariants of periodic time series, where the invariance follows directly from the continuity results. 
As a use-case, a simple machine learning experiment is conducted with real-world data, in which the suggested invariants improved the performance.
\end{abstract}



\maketitle

\section{Introduction}
\label{sec:intro}

The magnitude of a metric space \cite{Leinster2013} is a real number quantifying a certain sense of size. 
Based on the observation that every metric space can be identified as an enriched category, magnitude is obtained as a special case of Euler characteristic of enriched category, which encompasses generalization of Euler characteristic (along with M\"{o}bius inversion) of group, poset, and so forth \cite{Leinster2008}. 
Its theory is being developed through various tools such as differential calculus \cite{BarceloCarbery2018}, random matrix theory \cite{Meckes2020}, and homological algebra \cite{HepworthWillerton2017} and so forth. 
In addition, more classic quantitative notions such as cardinality, dimension \cite{Meckes2015}, and some intrinsic volumes \cite{GimperleinGoffeng2021} can be obtained as byproducts of magnitude by varying scale of a metric space. 

What does magnitude signify? 
Magnitude is considered ``the effective number of points'' \cite{Leinster2013} of a metric space in the sense that a group of points close to each other looks similar to a single point. 
Perhaps this peculiar character also motivated studies on the magnitude of various cases of metric spaces, in an attempt to understand the magnitude. 
Another possible attempt may be the investigation of continuity with respect to certain topology(e.g., Gromov-Hausdorff topology), as researchers have found both affirmative and negative cases. 
There is another concept relevant to the magnitude: weighting, a distribution that represents the contribution of each element of a metric space to the magnitude. 
Containing more information than the magnitude, the weighting must also be connected to the essence of magnitude. 

Meanwhile, Meckes \cite{Meckes2013} suggested a similar notion, maximum diversity, via modifying one of the definitions of magnitude, and showed that it not only shares similar properties with but also is closely related to magnitude. 
Unlike magnitude, maximum diversity lacks category-theoretic motivation, and not every feature of magnitude theory carries over to the maximum diversity side. 
For example, the counterpart for maximum diversity of weighting for magnitude has not been discussed. 
However, maximum diversity has advantages as well: for example, maximum diversity is continuous with respect to Gromov-Hausdorff distance, where the same property holds only partially for magnitude.

Understanding magnitude and weighting, or maximum diversity, which can be glimpsed through ecological model \cite{LeinsterCobbold2012} as well, seems significant in both theoretic and applicational domains, as its connection to machine learning and data analysis was observed \cite{Andreeva+2023}. 
Indeed, magnitude and weighting look to suit some tasks of data analysis, in that the dataset as a set of vectors can be regarded as finite metric space and the ``shape'', which magnitude or maximum diversity have information about, of such dataset may matter as TDA(topological data analysis) claims. 

Thus aiming for magnitude and relevant concepts, the current article contributes the following. 
\begin{itemize}
\item \highlight{%
This article explains how the continuity of weighting may implied from the continuity of magnitude, from an inequality that relates difference of weightings to difference of magnitudes.
}
\item \highlight{%
Also, ``maximum diversity weighting,'' which plays similar role to maximum diversity as weighting does to magnitude, is newly defined.. 
The inequality and continuity similar to those of weighting and magnitude are obtained for m.d. weighting and maximum diversity as well. 
}
\item \highlight{%
A new kind of numerical invariants of periodic time series is suggested. 
These invariants can be computed efficiently, in the sense that they do not require the input data be given with periodic segmentation. 
The nice properties originates from the continuity of magnitude and maximum diversity. 
A potential application of the invariants to machine learning and data analysis will be observed, specifically in electrocardiogram(ECG) identification.}
\end{itemize}


The remaining part of this article is organized as follows. 
In Section \ref{sec:background} we establish essential notions and review known results of magnitude theory. 
Section \ref{sec:continuity} contains the main theoretic contribution,  continuity of weighting, and its counterpart in the maximum diversity side. 
The following two sections discuss potential applications to time series analysis, Section \ref{sec:application} including propositions and simple examples, and Section \ref{sec:experiment} machine learning experiments.

Table \ref{tab:symbols} summarizes mathematical symbols those frequently used throughout this article and their names and references in this article.
\begin{table}
\centering
\begin{tabular}{ccc}
\hline 
symbol & name & reference  \\ \hline \hline
$\magn$ & magnitude & Definition \ref{def:weighting}(1) \\ \hline
$\wg$ & weighting & Definition \ref{def:weighting}(1) \\ \hline
$\magnp$ & maximum diversity & Definition \ref{def:weighting}(2) \\ \hline
$\wgp$ & maximum diversity weighting & Definition \ref{def:weighting}(2) \\ \hline
$Z_X$ & similarity matrix & Definition \ref{def:zeta}, equation \eqref{eq:zeta-map} \\ \hline
$\cW_X$ & weighting space & equation \eqref{eq:w-space} (see also Section \ref{subsec:w-general}) \\ \hline
$d_H$ & Hausdorff distance & equation \eqref{eq:d_H} \\ \hline
\end{tabular}
\caption{\highlight{Frequently used symbols.}}
\label{tab:symbols}
\end{table}

\section{Background on magnitude theory}
\label{sec:background}

In this section, we review definitions and properties of magnitude, weighting, and maximum diversity of metric spaces. 

A metric space $X=(X, d)$ is a set $X$ equipped with a real-valued function $d : X \times X \rightarrow \bR$ which satisfies the following three axioms: for each $x,y,z \in X$,
\begin{itemize}
    \item (nondegeneracy) $d(x,y) \geq 0$ where $d(x,y)=0 \,\Leftrightarrow \, x = y$
    \item (symmetricity) $d(x,y)=d(y,x)$
    \item (triangle inequality) $d(x,y)+d(y,z) \geq d(x,z)$
\end{itemize}
A finite metric space is a metric space whose underlying set is finite.

\highlight{%
We need the following concepts for the definitions magnitude and maximum diversity.
}

\begin{defn} \label{def:zeta}
\begin{enumerate}
    \item The \textbf{similarity matrix} $Z_X$ of a finite metric space $(X,d)$ is the square matrix
\begin{equation*}
    Z_X = (e^{-d(x,y)})_{x,y \in X}
\end{equation*}
with entries indexed by pairs of elements of $X$. 
\item A finite metric space $(X,d)$ is called \textbf{positive definite} if $Z_X$ is a (strictly) positive definite matrix. 
An arbitrary metric space $(X, d)$ is called positive definite if every finite subset of $X$ is positive definite.
\end{enumerate}
\end{defn}

Examples of positive definite metric space include $\bR^d$ for any natural number $d$ \cite{Leinster2013}, $L^p$ spaces ($0<p\leq 2$) \cite{Meckes2013}, weighted trees \cite{Hjorth+1998}, etc.

\highlight{%
Given a positive definite metric space $X$, its magnitude $\magn X$ and maximum diversity $\magnp X$ are both positive real numbers representing size of $X$ in certain senses. 
Both $\magn X$ and $\magnp X$ will be defined below for finite $X$ first and then extended to general case. 
The definitions of magnitude, weighting, and maximum diversity in this section are from works such as \mbox{\cite{Leinster2013,LeinsterMeckes2017}}
with slight modification.
}

\begin{rmk}
\highlight{%
The magnitude and weighting were originally defined more generally for enriched categories \mbox{\cite{Leinster2013}}. 
The exponential function in the definition of $Z_X$ is chosen due to its multiplicativity $e^{a+b} = e^a e^b$, which corresponds to a condition in the original definition.
In this paper, we will not cover the theory beyond metric spaces. 
}
\end{rmk}

\subsection{Finite subspace case}
Given a finite metric space $X$, let \textbf{weighting space} $\cW_X$ of $X$ be the set of signed measures on $X$, i.e.,
\begin{equation} \label{eq:w-space}
    \cW_X := \left\{ \sum_{x\in X} w_x \delta_x : w_x \in \bR \right\},
\end{equation}
where $\delta_x$ is the Dirac measure at $x$. 
We equip $\cW_X$ with the inner product 
\begin{align*}
    \langle \cdot,\cdot \rangle _{\cW_X} : \cW_X \times \cW_X &\rightarrow \bR \\
    \left\langle \sum_{x\in X} w_x \delta_x , \sum_{y\in Y} v_y \delta_y \right\rangle _{\cW_X} &= \sum_{x,y \in X} w_x v_y e^{-d(x,y)}.
\end{align*}
In addition, we define canonical pairing $\langle \cdot,\cdot \rangle$ of $\cW_X$ with the set $\bR^X$ of vectors indexed by elements of $X$ as
\begin{align*}
    \langle \cdot, \cdot \rangle : \cW_X \times \bR^X &\rightarrow \bR\\
    \left\langle \sum_{x \in X} w_x \delta _x , (v_x)_{x\in X}\right\rangle &= \sum _{x \in X} w_x v_x
\end{align*}
For notational convenience, we will canonically identify $\cW_X$ with the set $\bR^X$ via $\sum_{x \in X} w_x \delta_x = (w_x)_{x\in X}$. 
In particular, we have
\begin{equation*}
w,v \in \cW_X \quad \Rightarrow \quad \langle w, v \rangle_{\cW_X} = \langle w, Z_X v \rangle.
\end{equation*}

\begin{defn}\label{def:weighting}
Let $(X,d)$ be a finite, positive definite metric space.
\begin{enumerate}
\item 
\highlight{The \textbf{magnitude} $\magn(X)$ of $X$ is real number}
\begin{equation*}
\magn(X) = \langle {Z_X}^{-1} \mathbf{1}, \mathbf{1} \rangle,
\end{equation*}
\highlight{and the \textbf{weighting} $\wg(X)$ of $X$ is an element of $\cW_X$,}
\begin{equation*}
\wg(X) = {Z_X}^{-1} \mathbf{1}
\end{equation*}

\item Consider the same optimization problem with an additional constraint:
    \begin{equation}  \label{eq:opt-prob2}
        \begin{aligned}
        \underset{0 \neq w \in {\cW_X}}{ \text{\emph{maximize}} } \quad 
        & \frac{\langle w, \mathbf{1} \rangle^2}{\langle w, w \rangle_{\cW_X}}
        = \frac{\left(\sum_{x \in X} w_x \right)^2}{\sum_{x,y\in X} e^{-d(x,y)} w_x w_y},\\
        \text{\emph{subject to}} \quad 
        & w \succeq 0,
        \end{aligned}
    \end{equation}
    where $\succeq$ refers to the comparison between signed measures in $\cW_X$, or equivalently entry-wise comparison in $\bR^X$. 
    The \textbf{maximum diversity} $\magnp(X)$ of $X$ is the attained maximum, and $\wgp(X)$ is the solution to the problem which satisfies the normalization condition $\langle \wgp(X), \mathbf{1} \rangle = \magnp(X)$.
\end{enumerate}

\end{defn}
Some details need to be explained concerning the above definition. 
\paragraph{\highlight{\textbf{Alternative description for magnitude and weighting}}}{%
\highlight{$\magn(X)$ and $\wg(X)$ admit the following equivalent definitions \mbox{\cite[Proposition 2.4.3]{Leinster2013}}. 

Consider the following optimization problem on $w \in \cW_X$:}
    \begin{equation} \label{eq:opt-prob1}
    \underset{0 \neq w \in {\cW_X}}{ \text{{maximize}} } \quad 
     \frac{\langle w, \mathbf{1}\rangle ^2}{ \langle w, w \rangle_{\cW_X} } 
    = \frac{\left(\sum_{x \in X} w_x \right)^2}{\sum_{x,y\in X} e^{-d(x,y)} w_x w_y},
    \end{equation}
\highlight{where $\mathbf{1}$ is the all-one vector in $\bR^X$. 
Then, we can proved that $\magn(X)$ is equal to the attained maximum, and $\wg(X)$ of $X$ is the solution to the problem that satisfies the normalization condition $\langle \wg(X) ,\mathbf{1} \rangle = \magn(X)$. 
Indeed, the substitution $v = \frac{1}{\langle w, \mathbf{1} \rangle} w$ in \mbox{\eqref{eq:opt-prob1}} results in an equivalent optimization problem}
\begin{equation*}
\begin{aligned}
\underset{0 \neq v \in {\cW_X}}{ \text{maximize} } \quad 
& \frac{1}{ \langle Z_X v, v \rangle },\\
\text{subject to}\quad
& \langle v,\mathbf{1} \rangle =1,
\end{aligned}
\end{equation*}
\highlight{%
and the solution can be found e.g. by Lagrange multiplier method as in Definition \mbox{\ref{def:weighting}(1)}. 
This is a definition of magnitude $\magn$ parallel to that of maximum diversity $\magnp$.
}
}

\paragraph{\highlight{\textbf{Well-definedness of $\magnp$ and $\wgp$}}}{%
Since $Z_X$ is positive definite for finite, positive definite $X$ by definition, the objective function 
\begin{equation} \label{eq:obj-fn}
    \cW_X \ni w \neq 0
    \quad \rightarrow \quad 
    \frac{\langle w, \mathbf{1} \rangle^2}{\langle w, w \rangle_{\cW_X}}
    =\frac{\langle w, \mathbf{1} \rangle^2}{\langle w, Z_X w \rangle} \in \bR
\end{equation}
of the optimization problems \eqref{eq:opt-prob2} and \eqref{eq:opt-prob1} is well-defined for $w \neq 0$ and takes nonnegative values. 
Then the values $\magn(X)$ and $\magnp(X)$ are positive because the expression \eqref{eq:obj-fn} takes nonzero value for some $w$. 

Meanwhile, the uniqueness of $w_{X,+}$ can be observed as follows. 
If we substitute by $v = \frac{1}{\langle w, \mathbf{1} \rangle} w$ for $w$ such that $\langle w , \mathbf{1} \rangle \neq 0$, then maximizing \eqref{eq:obj-fn} is equivalent to minimizing $\langle v , v \rangle _{\cW}$, which is a strictly convex function by positive definiteness of $X$. 
Therefore, maximizer of problem \eqref{eq:opt-prob2} is uniquely determined up to a constant factor, where with respective normalization $\langle \wgp(X),\mathbf{1}\rangle = \magnp(X)$ $\wgp(X)$ is uniquely determined.
}

\paragraph{\highlight{\textbf{Heuristic explanations}}}{%
\highlight{Let us briefly observe how the magnitude $\magn(X)$ and the weighting $\wg(X)$ behaves.}
Since $Z_X\, \wg(X) = \mathbf{1}$, $\wg(X)$ means the contributions of individual points, which make up uniform value $1$ on each point when ``integrated'' with coefficients $Z_X = (e^{-d(x,y)})_{x,y\in X}$, whose entries are a decreasing function on the distance between points.
Then, for example, the entry of $\wg(X)$ of an outlier point will be close to $1$, because the point will neither give nor receive much contributions from other points. 
And \textit{vice versa}, a point with enough neighbor points can both give and receive contributions from neighboring points, so the corresponding entry of $\wg(X)$, in absolute value, may be much smaller than $1$. 
\highlight{Accordingly, $\magn(X)$ is expected to take larger values for $X$ with large mutual distances.

Meanwhile, the term ``maximum diversity'' is explained by the following interpretation of the function \mbox{\eqref{eq:obj-fn}} as a measure of diversity. 
Suppose that $w \in \cW_X$ satisfies $w \succeq 0$ and $\langle w, \mathbf{1} \rangle$ is fixed. 
Then the sum $\sum_{x,y \in X} w_x w_y$ is also fixed, and so the value of $\langle w,w\rangle_{\cW_X} = \sum_{x,y\in X} e^{-d(x,y)}w_x w_y$ becomes smaller if larger values of $w_x w_y$ correspond to larger distance $d(x,y)$. 
That is, the inverted quantity $\frac{1}{\langle w,w\rangle_{\cW_X}}$ measures how much the entries of $w$ are dispersed over $X$. 
}
}
\begin{ex}[originally from {\cite[p.1]{Willerton2015}}]
Let $X = \{x,y,z\} \subset \bR^2$ with $d(x,y)=d(x,z)=t$, $d(y,z)=ct$ for a positive real number $t$ and real number $c$ such that $0< c < 2$.
If we enumerate entries indexed by $x,\,y,\,$ and $z$ in this order,  the similarity matrix becomes 
\begin{equation*}
    Z_X = \begin{pmatrix}
    1 & e^{-t} & e^{-t} \\
    e^{-t} & 1 & e^{-ct} \\
    e^{-t} & e^{-ct} & 1
\end{pmatrix},
\end{equation*}
and its inverse is 
\begin{align*}
    {Z_X}^{-1}  = & \frac{1}{(1+2e^{-(2+c)t})-(2e^{-2t}+e^{-2ct})} \times \\
& \begin{pmatrix}
    1-e^{-2ct} & {e^{-(1+c)t}-e^{-t}} & {e^{-(1+c)t}-e^{-t}} \\
    {e^{-(1+c)t}-e^{-t}}& 1-e^{-2t} & e^{-2t}-e^{-ct} \\
    {e^{-(1+c)t}-e^{-t}} & e^{-2t}-e^{-ct}& 1-e^{-2t}
\end{pmatrix}. 
\end{align*}
Thus, we have
\begin{align*}
\wg(X) & = {Z_X}^{-1} \mathbf{1} \\
& = \frac{1}{1+e^{-ct}-2e^{-2t}} 
\begin{pmatrix}
    1+e^{-ct}-2e^{-t}\\ 
    1-e^{-t}\\ 
    1 -e^{-t}
\end{pmatrix}, \\
\magn(X) & = \frac{3 +e^{-ct}-4e^{-t}}{1+e^{-ct}-2e^{-2t}}.
\end{align*}
Since the entries of $\wg(X)$ are positive, we have $\magnp(X) = \magn(X)$ and $\wgp(X)=\wg(X)$. 
Here are some immediate but noteworthy observations.
\begin{enumerate}
    \item $\lim_{t \searrow 0} \magn(X) = 1$ 
\item $\lim_{t \searrow 0} \wg(X)_x = \frac{2-c}{4-c},\quad \lim_{t\searrow 0} \wg(X)_y = \lim_{t\searrow 0} \wg(X)_z = \frac{1}{4-c}$
\item $\lim_{t \rightarrow \infty} \magn(X) = 3$
\item $\lim_{t\rightarrow \infty} \wg(X)_x = \lim_{t\rightarrow \infty} \wg(X)_y =\lim_{t\rightarrow \infty} \wg(X)_z = 1$
\item $\wg(X)_x - (\wg(X)_y + \wg(X)_z) = \frac{-(1-e^{-ct})}{1+e^{-ct}-2e^{-2t}} = -\frac{c}{2-c}t + O(t^2)$
\end{enumerate}
In the words of the heuristic explanation above: when $t$ is small, \highlight{the three points contribute to one another} significantly, so they divide up total value $1$ into parts (item (1)); when $t$ is large, the contributions are meager, so each point almost maintains its value $1$. (item (3) and (4))
Moreover, when $t$ is \emph{not} large enough so that $y$ and $z$ stay relatively close to each other, $\wg(X)_x$ is close to $\wg(X)_y+\wg(X)_z$, as the contributions of the two points $y$ and $z$ is similar to that of $x$. (item (2) and (5))

The description ``the effective number of points'' for the magnitude $\magn$ of \cite{Leinster2013} also expresses how weighting (when summed up to magnitude) moderates the role of each point of the given set. 
\end{ex}
\begin{figure}
  \centering
  \includegraphics[width=0.7\textwidth]{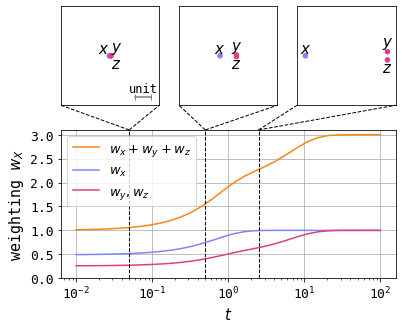}
  \caption{\highlight{Magnitude $\magn(X)$ and weighting $\wg(X)$ of three-point metric space} $X=\{x,y,z\}$ with $d(x,y)=d(x,z)=t,\ d(y,z)=0.1t$.}
  \label{fig:example-3pts}
\end{figure}

\begin{ex} \label{ex:double-s}
    For $i=1,\,2$, let $X_i$ be a finite subset of $\bR^2$ which consists of $N_i$ points randomly sampled with Gaussian noise of variance $\sigma_i$ from S-shaped parametric curve\footnote{This is implemented with Scikit-learn \cite{python.sklearn}, using \texttt{sklearn.datasets.make\_s\_curve} .}
    \begin{equation*}
        (\sin t ,\, \mathrm{sgn} (t) \cdot (\cos t - 1) ) \quad \left(  -\frac{3}{2} \pi \leq  t \leq \frac{3}{2}\pi \right) 
    \end{equation*}
in $\bR^2$. 
If $N_i$'s are large enough and $\sigma_i$'s are small enough, $X_1$ and $X_2$ will look ``effectively'' similar to each other in the sense explained above, as illustrated by the following calculation.

Suppose each of the weightings $\wg({X_1})$ and $\wg({X_2})$ as signed measures is convoluted with a kernel function. 
The result of convolution, which aggregates contributions of neighboring points, will be similar to each other. 
A specific case, including {\dispersion}s $\wgp({X_1})$ and $\wgp({X_2})$, is visualized in Figure \ref{fig:example-bna}, and the result indicates how ``effective'' count of magnitude and maximum diversity is distributed by the weighting and {\dispersion}.
\end{ex}
\begin{figure*}
  \centering
  \includegraphics[width=0.98\textwidth]{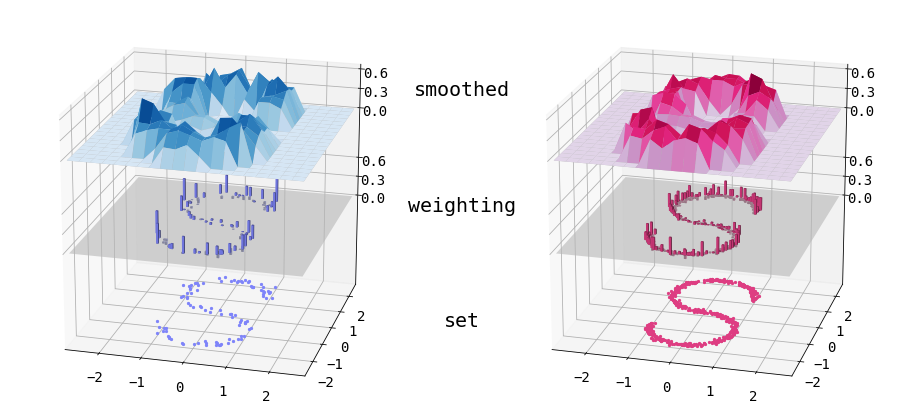}
  \includegraphics[width=0.98\textwidth]{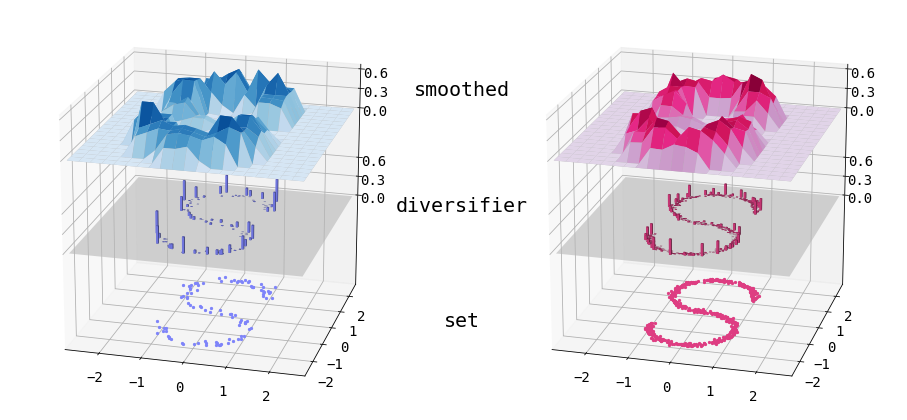}
  \caption{Weightings and {\dispersion}s of two finite subsets with similar shapes in Example \ref{ex:double-s}. 
  \highlight{The top layers of each frame shows the convolution of $\wg(X_i)$ or $\wgp(X_i)$ with a kernel function.}
  The parameters are $N_1=100,\,\sigma_1 = 0.1$ (left) and $N_2=500,\,\sigma_2=0.05$ (right), and the kernel is the indicator function of the disk of radius $0.5$. }
  \label{fig:example-bna}
\end{figure*}

\highlight{The following properties are} immediate from Definition \ref{def:weighting}. 
\begin{prop} \label{prop:basics}
\begin{enumerate}
    \item (monotonicity) For two finite positive definite metric spaces $A$ and $B$, we have
    \begin{equation} \label{eq:monotone-f}
    A \subset B \quad \Rightarrow \quad \magn(A) \leq \magn(B),\, \magnp(A) \leq \magnp(B)
    \end{equation}
    \item For any finite positive definite metric space $X$, we have
\begin{align}\label{eq:m=w1=ww-f} 
    \magn(X)_{\phantom{+}} & = \langle \wg(X) ,_{\phantom{+}}\mathbf{1} \rangle = \langle \wg(X) ,_{\phantom{+}} \wg(X)_{\phantom{+}} \rangle_{\cW_X} , \\ \label{eq:m=d1=dd-f}
    \magnp(X) & = \langle  \wgp(X), \mathbf{1} \rangle = \langle \wgp(X) , \wgp(X) \rangle_{\cW_X}.
\end{align}
\end{enumerate}
\end{prop}

\begin{rmk} \label{rmk:1}
The entity $\wgp({X})$ is neither given a name nor defined in the literature, and we will call it \textbf{maximum diversity weighting}, \highlight{or \mbox{\dispersion} in short, in this paper}.
In \cite{LeinsterMeckes2017} not $\wgp({X})$ itself but its normalization to probability distribution is discussed and referred to as ``maximizing distribution.'' 
\highlight{The ``maximum diversity weights'' in \mbox{\cite[\S 2.2]{HeFong2019}} has no direct relation to $\wgp$ of the current article, despite seemingly relevant name.} 
\end{rmk}

\subsection{General case} \label{subsec:w-general}
Magnitude and maximum diversity are also \highlight{defined for compact positive metric spaces in general}. 
From now on, we fix a (possibly noncompact) positive definite metric space $X$ and consider compact subsets of $X$. 

\begin{defn}
Let $A$ be a compact, positive definite metric space. 
Then the \textbf{magnitude} $\magn(A) \in [0, +\infty]$ is defined as
\begin{align*}
    \magn(A) 
    & := \sup \{ \magn(B) :  B \subset A ,\, \#B < \infty \} \in [0, +\infty]
\end{align*}
Likewise, the \textbf{maximum diversity} $\magnp(A)$ is defined as
\begin{equation*}
    \magnp(A) := \sup \{ \magnp(B) :  B \subset A ,\, \#B < \infty \} \in [0, +\infty].
\end{equation*}
\end{defn}
From the definition, the monotonicity in Proposition \ref{prop:basics} for both magnitude and maximum diversity follows immediately.

To extend the definition of weighting to the general case, we need to define the weighting space $\cW_X$ appropriately. 
In addition, the weighting space will also be defined for noncompact $X$ for later use. 
The following construction is introduced in \cite{Meckes2015}.

Let $\cH_X = (\cH_X , \langle \cdot ,\cdot \rangle _{\cH_X})$ be the RKHS(reproducing kernel Hilbert space) with respect to the continuous kernel function $K(x,y) := e^{-d(x,y)}$. 
This means that $\cH_X$ consists of continuous real-valued functions on $X$, which is equal to the closed span of $\{K(x,\cdot): x \in X\}$ and satisfies reproducing property $\langle h, K(x,\cdot) \rangle_{\cH_X} =h(x)$ for each $h \in \cH_X$. 
Again, positive definiteness of $X$ assures the inner product $\langle \cdot, \cdot \rangle_{\cH_X}$ to be positive definite. 

Let
    \begin{equation*}
    FM_X = \left\{ \sum_{i=1}^k m_i \delta_{x_i} :  m_i \in \bR,\, x_i \in X \right\}
    \end{equation*}
be the set of signed measures with finite support, where $\delta_x$ refers to the Dirac measure on $x$ as before. 
Since the map $Z_X : FM_X \rightarrow \cH_X$ defined by 
\begin{equation}\label{eq:zeta-map}
Z_X \left( \sum_{i=1}^k m_i \delta_{x_i} \right) (x) = \sum_{i=1}^k m_i e^{-d(x_i, x)}.
\end{equation}
is injective by positive definiteness of $X$, we can equip $FM_X$ with the inner product $\langle \cdot , \cdot \rangle _{\cW_X}$ pulled back by $Z_X$. 
Now, let the weighting space $\cW_X$ of $X$ be the completion of $FM_X$ with respect to $\langle \cdot , \cdot \rangle_{\cW_X}$. 
Then the extension $Z_X:\cW_X \rightarrow \cH_X$ is still an isometry. 

Meanwhile, since every element of $\cH_X$ is a continuous real-valued function on $X$, we have restriction $\langle \cdot, \cdot \rangle : FM \times \cH_X \rightarrow \bR$ of canonical pairing between measures and continuous functions. 
The relation $\langle w, h \rangle = \langle Z_X w, h \rangle _{\cH_X}$ extends continuously to $\langle \cdot, \cdot \rangle : \cW_X \times \cH_X \rightarrow \bR$. 
It is easy to observe that 
\begin{equation*}
u, v \in \cW_X
\quad
\Rightarrow
\quad
\langle u, v \rangle _{\cW_X} = \langle u, Z_X v\rangle = \langle Z_X u, Z_X v \rangle _{\cH_X}
\end{equation*}

\begin{defn}
Let $X$ be a compact positive definite metric space. 
Then the \textbf{weighting} $\wg(X)$ of $X$ is the element of $\cW_X$ such that $Z_X \wg(X) = 1$ on $X$. 
\end{defn}
This definition indeed generalizes Definition \ref{def:weighting}(1). 
It is shown in \cite[Theorem 3.4]{Meckes2015} that the weighting $\wg(X)$ of $X$ exists if and only if $\magn(X) < \infty$ and that $\magn(X) = \langle \wg(X), \wg(X) \rangle_{\cW_X}$ in the case. 
At this moment, there is one more equality $\magn(X) = \langle \wg(X) , \mathbf{1} \rangle$ in equation \eqref{eq:m=w1=ww-f} to establish, which will be formulated in slightly more general form for later use.

When discussing the weighting of a compact subset $A$ of positive definite metric space $X$, it is convenient to understand $w_A$ as an element of $\cW_X$. 
This is possible because the inclusion map $FM_A \rightarrow FM_X$, which is an isometry with respect to $\langle \cdot ,\cdot \rangle _{\cW_A}$ and $\langle \cdot ,\cdot \rangle _{\cW_X}$, uniquely extends to an injection $\cW_A \rightarrow \cW_X$. 
\begin{prop}[{\cite[Proposition 4.2]{Meckes2015}}] \label{prop:wh=|A|}
If $A$ is a compact subset of a positive definite metric space $X$, then $\langle \wg(A), h \rangle = \magn(A)$ for any $h \in \cH_X$ such that $h=1$ on $A$.
\end{prop}

\begin{proof}
Since $\wg(A) \in \cW_A \subset \cW_X$, there exists a sequence $(w_n)$ in $FM_A$ that converges to $\wg(A)$. 
Then we have
\begin{equation*}
\begin{aligned}
\langle \wg(A) ,h \rangle 
& = \lim_{n \rightarrow \infty} {\langle w_n ,h \rangle }  \\
& = \lim_{n \rightarrow \infty} {\langle w_n, Z_A \wg(A) \rangle  } && (Z_Aw_A = h = 1 \text{ on } A)\\
& = \langle \wg(A) , Z_A \wg(A) \rangle  = \magn(A).
\end{aligned}
\end{equation*}
\end{proof}

\begin{rmk}\label{rmk:R^d}
When $X=\bR^d$ is equipped with Euclidean distance, the weighting space $\cW_X$ is (up to a constant factor) isometric to the Sobolev space $H^{-\frac{d+1}{2}}(\bR^d)$ \cite{Meckes2015}. 
In particular, each weighting $\wg(A)$ of compact subset $A$ of $\bR^d$ is in fact a distribution.
This also shows that weighting need not be a measure. 
\end{rmk}

Before closing this section, we review relevant properties of magnitude and weighting. 
Let $\mathcal{C}_X$ be the collection of compact subsets of $X$, equipped with Hausdorff distance $d_H$. 
Recall that the Hausdorff distance $d_H(A,B)$ between two subsets $A$ and $B$ of a metric space $X$ is defined as
\begin{equation} \label{eq:d_H}
d_H(A,B) = \min \{  \max\{d(a,B):a\in A \} ,\, \max\{d(b, A): b \in B \} \},
\end{equation}
where
\begin{equation}
    d(x,A) = \min \{d(x,a) : a \in A \}
\end{equation}
as usual. 
The following propositions are essentially from \cite[Theorem 2.6 \& Proposition 2.11]{Meckes2013}, which do not explicitly mention the estimate below but contain sufficient arguments for it.

\begin{prop} \label{prop:mm-conti}
Let $X$ be a positive definite metric space and $\mathcal{C}_X$ the collection of compact subsets of $X$, equipped with Hausdorff distance $d_H$.
\begin{enumerate}
\item Magnitude as a real-valued function $\magn :\mathcal{C}_X \rightarrow \bR$ is lower semicontinuous. 
\item Maximum diversity as a real-valued function $\magnp:\mathcal{C}_X \rightarrow \bR$ is continuous. 
Specifically, for any $A,B \in \mathcal{C}_X$ we have
\begin{equation*}
1 - 4\magnp(A)\, d_H(A,B) 
\leq \frac{\magnp(B)}{\magnp(A)}   
\leq  \frac{1}{1-4\magnp(A)\, d_H(A,B)}
\end{equation*}
\end{enumerate}
\end{prop}

\begin{proof}[\highlight{Proof}]
Let $A, B \in \mathcal{C}_X$ be given. 
From inequalities in the proof of \cite[Theorem 2.6]{Meckes2013}, there exist positive measures $\mu,\nu$ on $X$ such that 
\begin{align*}
\magnp(B) 
& \geq \frac{\nu(B)^2}{\langle \nu, \nu \rangle _{\cW_X}} \\
& \geq \frac{\mu(A)^2}{\langle \mu, \mu \rangle _{\cW_X} + 4 \mu(A)^2\, d_H (A,B)} \\
& \geq \frac{1}{1+\varepsilon}\frac{\langle \mu, \mu \rangle _{\cW_X}\cdot\magnp(A)}{\langle \mu, \mu \rangle _{\cW_X} + 4 \mu(A)^2 \, d_H (A,B)}.
\end{align*}
Since $\langle \mu, \mu \rangle _{\cW_X} \leq \mu(A)^2$ and $\varepsilon > 0$ can be chosen arbitrarily, 
\begin{equation*}
\magnp(B) 
\geq  \frac{\magnp(A)}{1 + 4 \frac{\mu(A)^2}{\langle \mu, \mu \rangle _{\cW_X}} d_H (A,B)} 
\geq \frac{\magnp(A)}{1 + 4 \magnp(A) d_H (A,B)} 
\end{equation*}
Likewise, the inequality
\begin{equation*}
\magnp(B) \leq \frac{\magnp(A)}{1-4\magnp(A)d_H(A,B)}
\end{equation*}
is from \cite[Proposition 2.11]{Meckes2013}. 
\end{proof}

\begin{rmk}
There exists a counterexample \cite[Example 2.3]{GimperleinGoffengLouca2025} to the upper continuity of magnitude with respect to Hausdorff distance. 
On the other hand, magnitude on compact subsets of $\bR^d$ is upper continuous at each compact subset which is star-shaped about a point of its interior \cite[Theorem 5.4.15]{LeinsterMeckes2017}. 
\end{rmk}

\section{Continuity of weighting and {\dispersion}}
\label{sec:continuity}

From the equations \eqref{eq:m=w1=ww-f}-\eqref{eq:m=d1=dd-f}, $\langle \wg(X) , \wg(X) \rangle _{\cW_X} = \magn(X)$ (resp. $\langle \wgp({X}), \wgp({X}) \rangle_{\cW_X}=\magnp(X)$) for finite $X$, we immediately observe that, as a function on $\mathcal{C}_X$ the continuity of weighting $\wg(X)$ (resp. {\dispersion} $\wgp({X})$) cannot hold without that of magnitude $\magn(X)$ (resp. maximum diversity $\magnp(X)$). 
A converse will be dealt with in this section: the continuity of magnitude $\magn(X)$ (resp. maximum diversity $\magnp(X)$) implies the continuity of weighting $\wg(X)$ (resp. {\dispersion} $\wgp({X})$). 

\subsection{Weighting case}

\highlight{
Let $\| \cdot \|_{\cW_X}$ be the norm induced by inner product $\langle \cdot, \cdot \rangle _{\cW_X}$.
}
\begin{lem}\label{lem:wm-ineq}
For two compact subsets $A$ and $B$ of a positive definite metric space $X$, we have
\begin{equation*}
\| \wg(A) - \wg(B) \|  _{\cW_X}
\leq \sqrt{\magn(A\cup B) - \magn(A)} 
+ \sqrt{\magn(A\cup B) - \magn(B)} .
\end{equation*}
\end{lem}
\begin{proof}
Comparing the weightings of $A$ and $A \cup B$, we have
\begin{align*}
\| \wg(A) - \wg({A\cup B}) \|_{\cW_X}^2  
& = \langle \wg(A), \wg(A) \rangle_{\cW_X} 
+ \langle \wg({A\cup B}), \wg({A\cup B}) \rangle_{\cW_X}  \\
&\phantom{{}={}}- 2\langle \wg({A\cup B}), \wg(A)\rangle_{\cW_X} \\
& = \magn(A) + \magn(A\cup B) - 2 \langle \wg(A) , Z_X\wg({A \cup B})\rangle \\ 
& = \magn(A) + \magn({A\cup B}) - 2\magn(A) \\ 
& = \magn({A\cup B}) - \magn(A) ,
\end{align*}
by equation \eqref{eq:m=w1=ww-f} and Proposition \ref{prop:wh=|A|} because $Z_X \wg({A \cup B}) = 1$ on $A \cup B$ and thus on $A$ in particular. 
Combining with the same inequality for $Y$ and $X \cup Y$, we deduce
\begin{align*}
\| \wg(A) - \wg(B) \|_{\cW_X} 
& \leq \| \wg(A) - \wg({A\cup B}) \|_{\cW_X} + \| \wg({A\cup B}) - \wg({B}) \|_{\cW_X} \\
& = \sqrt{\magn(A\cup B) - \magn(A)} + \sqrt{\magn(A \cup B) - \magn(B)}.
\end{align*}
\end{proof}

\begin{prop}
Let $A$ be a compact subset of a locally compact positive definite metric space $X$, and magnitude $\magn$ is continuous at $A$ \highlight{with respect to Hausdorff distance}. 
Then weighting $\wg : \mathcal{C}_X \rightarrow \cW_X$ is also continuous at $A$ \highlight{with respect to Hausdorff distance}. 
\end{prop}
\begin{proof}
Let a positive real number $\epsilon$ be given. 
By local compactness, there exists a compact neighborhood $D$ containing $A$, i.e. $D$ is compact and there exists an open subset of $X$ between $A$ and $D$. 
By upper continuity of magnitude at $A$, we may assume $\magn(D) < \magn(A) + \epsilon^2$ by shrinking $D$ if needed. 
Let $\delta_1$ be a positive real number such that $\delta_1$-neighborhood of $A$ in $X$ is contained in $D$. 
Let $\delta_2$ be a positive real number such that, by lower continuity of magnitude,
\begin{equation*}
B \in \mathcal{C}_X \text{ and } d_H(A, B) < \delta_2
\quad \Rightarrow \quad 
\magn(B) > \magn(A) - \epsilon^2.
\end{equation*}
Then for any $B \in \mathcal{C}_X$ such that $d_H(A,B) < \min \{\delta_1, \delta_2\}$, we have
\begin{align*}
\|\wg(A) &- \wg(B) \| _{\cW_X} \\
&\leq  \sqrt{\magn(A \cup B) - \magn(A)} + \sqrt{\magn(A \cup B) - \magn(B)} \\
& \leq \sqrt{ \magn(D) - \magn(A)} + \sqrt{\magn(D) - \magn(B)} \\
& = \sqrt{ \magn(D) - \magn(A) } + \sqrt{ (\magn(D) - \magn(A)) + (\magn(A) - \magn(B))} \\
& \leq \sqrt{\epsilon^2} + \sqrt{ \epsilon^2 + \epsilon^2 } < 3 \epsilon
\end{align*}
because $A \cup B \subset D$. 
The proof is complete. 
\end{proof}
By \cite[Theorem 5.4.15]{LeinsterMeckes2017}, when $X=\bR^d$ in particular, weighting is continuous at compact subsets of $\bR^d$ which is star-shaped with respect to a point of its interior. 

\subsection{Diversifier case}
Despite the stronger continuity property of maximum diversity than magnitude, we need additional preparation to adapt the proof in the previous subsection. 
First, we need to define {\dispersion} for general compact positive definite metric space. 
Second, we need to establish a property analogous to $Z_X \wg(X) = \mathbf{1}$ for maximum diversity. 
It turns out that the two goals are intertwined.

Note that unlike the counterpart $Z_X \wg(X) = \mathbf{1}$ for weighting, the following lemma is insufficient to fully characterize {\dispersion} $\wgp({X})$ of given $X$. 
\begin{lem}\label{lem:Zw+=1+}
If $X$ is a finite positive definite metric space, then we have $Z_X \, \wgp({X}) \succeq \mathbf{1}$, i.e. every entry of $Z_X \, \wgp({X})$ is greater than or equal to $1$, and at least one entry is equal to $1$.
\end{lem}

\begin{proof}
Let us substitute $\frac{1}{\langle w, \mathbf{1} \rangle}w$ in optimization problem in Definition \ref{def:weighting} by $v = (v_x)_{x \in X}$ and consider equivalent optimization problem
    \begin{align*}
    \underset{v \in \bR^X}{ \text{minimize} }\quad
    & \langle v, Z_X v \rangle \\
    \text{subject to}\quad
    & \langle v, \mathbf{1} \rangle = 1, \\
    & -v \preceq 0,
    \end{align*}
on $v$. 
Then the Lagrangian of this problem is
    \begin{equation*}
    L(v; \mu , \lambda) 
    = \langle v, Z_X v \rangle + \langle \mu, -v\rangle + \lambda \langle v, \mathbf{1} \rangle
    \quad (\mu=(\mu_x)_{x\in X} \succeq 0),
    \end{equation*}
and by Karush-Kuhn-Tucker theorem a minimizing $v$ satisfies
    \begin{equation*}
    2 Z_X v - \mu + \lambda \mathbf{1} = 0.
    \end{equation*}
where $\mu_x = 0$ if $v_x > 0$.
Returning to $w$ to obtain
    \begin{equation*}
    2 Z_X w - \langle w, \mathbf{1} \rangle \mu + \lambda \langle w, \mathbf{1} \rangle \mathbf{1} = 0,
    \end{equation*}
and taking dot product with $w$ we get in particular
    \begin{equation*}
    2 \langle w, Z_X w \rangle - \lambda \langle w, \mathbf{1} \rangle^2 = 0,
    \end{equation*}
because $\mu_x = 0$ if $w_x > 0$. 
Since $\langle w ,Z_X w\rangle = \langle w, \mathbf{1} \rangle$ by hypothesis we have $\lambda = \frac{2}{\langle w, \mathbf{1} \rangle}$, which leads to
    \begin{equation*}
    Z_X w = \mathbf{1} + \frac{\langle w, \mathbf{1} \rangle}{2} \mu. 
    \end{equation*}
Here some entries of $\mu$ are zero because $\langle w, \mathbf{1} \rangle = 1$.
\end{proof}

\begin{lem}\label{lem:dm-ineq-f}
For finite subsets $A$ and $B$ of a positive definite metric space $X$, we have    
\begin{multline*}
\| \wgp({A}) - \wgp({B}) \|_{\cW_X} \\
\leq \sqrt{\magnp(A\cup B) - \magnp(A)} + \sqrt{\magnp(A \cup B) - \magnp(B)}.
\end{multline*}
\end{lem}

\begin{proof}
The proof proceeds similarly as in the proof of Lemma \ref{lem:wm-ineq}. 
Only the equality $\langle Z_X \wg({A \cup B}), \wg(A) \rangle = \magn(A)$ need to be replaced by the inequality $\langle Z_X \wgp({A \cup B}), \wgp({A}) \rangle \geq \magnp(A)$, which can be obtained from Lemma \ref{lem:Zw+=1+} and the positivity of {\dispersion} $\wgp({A})$. 
\end{proof}

Now we can give an alternative definition of weighting, and the extended definition of {\dispersion}. 
\highlight{%
Recall that the concept of \textbf{net} generalizes that of sequence in a topological space
(see books on general topology, e.g., \mbox{\cite[{\S}11]{Willard2004}}).
}
Let $\cF_A$ denote the collection of finite subsets of $A$. 

\begin{prop} \label{prop:w-net}
Let $A$ be a compact subset of a positive definite metric space $X$. 
\begin{enumerate}
    \item The net $(\wg(B))_{B \in \cF_A}$ of weighting of finite subsets of $A$ converges in $\cW_X$ if and only if $\magn(A) < \infty$, in which case the limit $w \in \cW$ satisfies $Z_X w = 1$ on $A$. 
    \item The net $(\wgp({B}))_{B \in \cF_A}$ of {\dispersion} of finite subsets of $A$ converges in $\cW_X$ if and only if $\magnp(A) < \infty$, in which case the limit $w_+ \in \cW$ satisfies $Z_X w_+ \geq 1$ on $A$. 
\end{enumerate}
\end{prop}

\begin{proof}
We prove (1) only because (2) can be proved similarly.

If the net $(w_B)_{B \in \cF_A}$ converges in $\cW_X$, then $\langle \wg(B),\wg(B) \rangle_{\cW} = \magn(B)$ is bounded and hence $\magn(A) = \sup \{ \magn(B): B \in \cF_A \} < \infty$. 
Conversely, suppose $\magn(A)<\infty$ and a positive real number $\epsilon$ is given. 
Then, \highlight{by the definition of $\magn$ in general case}, there exists a \textit{finite} subset $B$ of $A$ such that $\magn(B) > \magn(A) - \epsilon^2$. 
For any pair $B_1, B_2$ of finite subsets of $A$ such that $B_1 , B_2 \supset B$, we have from Lemma \ref{lem:wm-ineq}, 
\begin{align*}
    \|\wg({B_1}) - \wg({B_2}) \| _{\cW} 
    & \leq \sqrt{\magn(B_1 \cup B_2 ) - \magn(B_1)} + \sqrt{\magn(B_1 \cup B_2 ) - \magn(B_2)} \\
    & \leq \sqrt{\magn(A) - \magn(B)} + \sqrt{\magn(A) - \magn(B)}  \leq 2 \epsilon
\end{align*}
hence the net $(\wg(B))_{B \in \cF_A}$ is Cauchy and so convergent because $\cW_X$ is a Hilbert space. 
Since $Z_X \wg(B) = 1$ on $B$ and the convergence in $\cH_X$ implies pointwise convergence by a general property of RKHS, the limit $w$ also satisfies $Z_X w=1$ on $B$.
\end{proof}
As a result of Proposition \ref{prop:w-net} we can define the {\dispersion} of (possibly infinite) compact positive definite metric space $X$ as the limit of {\dispersion} of finite subsets:
\begin{equation*}
\wgp({X}) := \lim _{Y \in \cF_X} \wgp({Y})
\end{equation*}

Let us check the properties \eqref{eq:m=d1=dd-f} of {\dispersion}. 
The equality $\magnp(X) = \langle \wgp(X), \wgp(X) \rangle _{\cW_X}$ follows from taking the limit of the same inequality for the finite subset case. 
Next, the counterpart of Proposition \ref{prop:wh=|A|} can be proved in similar way as follows. 

To elaborate, we need to prove $\magnp(A) = \langle \wgp({A}), h \rangle$ for compact subset $A$ of positive definite metric space $X$ and $h \in \cH_X$ such that $h=1$ on $A$. 
If $B$ is a finite subset of $A$, then we have $\langle \wgp({B}), h \rangle = \magnp(B)$ by equation \eqref{eq:m=d1=dd-f} because $h=1$ on $B$ as well. 
By taking the limit on $B$ over finite subsets of $A$ the proof is complete. 

The inequality $Z_X\wgp({A}) \geq 1$ has been proved in Proposition \ref{prop:w-net}. 

\begin{prop}\label{prop:dm-ineq}
For two compact subsets $A$ and $B$ of a positive definite metric space $X$, we have
\begin{multline*}
\| \wgp({A}) - \wgp({B})\|_{\cW_X} \\
 \leq \sqrt{\magnp(A\cup B) - \magnp(A)} 
+ \sqrt{\magnp(A \cup B) - \magnp(B)}.
\end{multline*}
\end{prop}

\begin{proof}
This follows from Lemma \ref{lem:dm-ineq-f} by taking limit over finite subsets, because of the definition of {\dispersion} in the general case and the continuity of maximum diversity (Proposition \ref{prop:mm-conti}).
\end{proof}

\begin{prop} \label{prop:mw-diff-estimate}
Let $A$ be a compact subset of a positive definite metric space $X$, and $\magnp$ be continuous at $A$. 
Then {\dispersion} $w_{-,+} : \mathcal{C}_X \rightarrow \cW_X$ is also continuous at $A$.  
Specifically, for two compact subsets $A$ and $B$ of $X$, we have
\begin{equation} \label{eq:dm-dH-ineq}
\| \wgp({A})-\wgp({B})\|_{\cW_X} \leq 4 \sqrt{2}\, \magnp( A \cup B) \, d_H(A,B)^{\frac{1}{2}}
\end{equation}
\end{prop}
\begin{proof}
The proof is straight from Proposition \ref{prop:mm-conti} and \ref{prop:dm-ineq}. 
Indeed, the inequality \eqref{eq:dm-dH-ineq} follows from
\begin{align*}
\| \wgp({A}) &- \wgp({B}) \|  \\
& \leq \sqrt{\magnp(A\cup B) - \magnp(A)} + \sqrt{\magnp(A \cup B) - \magnp(B)} \\
& = \sqrt{ \magnp(A \cup B) } \cdot \left( \sqrt{1 - \frac{\magnp(A)}{\magnp(A\cup B) }} + \sqrt{1-\frac{\magnp(B)}{\magnp(A\cup B)}}\right) \\
& \leq \sqrt{ \magnp(A \cup B) } \cdot \\
& \qquad \left( \sqrt{ 8 \magnp(A\cup B) d_H(A,A\cup B)} + \sqrt{ 8\magnp(A\cup B) d_H(B,A\cup B)}\right)  \\
& \leq \sqrt{ \magnp(A \cup B) } \cdot \\
& \qquad \left( \sqrt{ 8\magnp(A\cup B) d_H(A,B)} + \sqrt{ 8\magnp(A\cup B) d_H(A,B)}\right)  \\
& = 4 \sqrt{2} \, \magnp(A \cup B) \, \sqrt{d_H(A,B)}
\end{align*}
\end{proof}

From the above proposition we immediately obtain a large family of continuous functions on $\cC_X$ as follows.
\begin{cor} \label{cor:1}
For each $h \in \cH_X$, the function $\varphi_h:\cC_X \rightarrow \bR$ defined by $\varphi_h(A) := \langle \wgp({A}), h \rangle$ is continuous with respect to Hausdorff distance. 
\end{cor}

\begin{rmk}
\highlight{%
In the light of data analysis, we may consider $\varphi_h$ of Corollary \mbox{\ref{cor:1}} an invariant for point clouds (by setting $A \in \mathcal{C}_X$ to be finite), as alluded to in Example \mbox{\ref{ex:double-s}}. 
Point cloud invariants are also studied in \mbox{\cite{AnosovaKurlin}}, but the invariants treated in \mbox{\cite{AnosovaKurlin}} are not continuous with respect to Hausdorff distance unless the cardinality is fixed. 
On the other hand, there are other point cloud invariants used in TDA(topological data analysis) studies (e.g., \mbox{\cite{SDB2016}}).
} 
\end{rmk}

For the sake of completeness, we reformulate a fact about {\dispersion} previously known in \cite[Proposition 2.9]{Meckes2015}.

\begin{prop}
For a compact, positive definite metric space $X$, its {\dispersion} $w_{X,+}$ is a (nonnegative) measure on $X$. 
\end{prop}

\begin{proof}
Let $(\wgp({B}))_{B \in \cF_X}$ be the net of the {\dispersion s} of finite subsets of $X$. 
For each $B \in \cF_X$ we have
\begin{equation*}
    \langle \wgp({B}), \mathbf{1}_B  \rangle = \magnp(B) \leq \magnp(X)
\end{equation*}
This means that the set $\{\wgp({B}): B \in \cF_X \}$, as a subset of nonnegative measures on $X$, is bounded in total variation norm. 
Since each measure on $X$ is an element of dual space of $C(X)$, by Alaoglu's theorem, the net $(\wgp({B}))_{B \in \cF_X}$ possesses a convergent subnet. 
But we have, for any $\mu \in \cW_X$ which is a measure on $X$ as well,
\begin{equation*}
\langle \mu, \mu \rangle _{\cW} \leq  (\text{total variation of } \mu)^2
\end{equation*}
which means that the convergence as measures implies that as elements of $\cW$ as well. 
Therefore the limit of the subnet must be $\wgp({X})$.
\end{proof}

\section{A new invariant of periodic time series}
\label{sec:application}

As an application, we will suggest a family of invariants of periodic time series satisfying the following condition: 
\begin{quote}
Given periodic time series with period $T$, the invariant calculated from entire time series is equal to the invariant calculated from its restriction to any interval with duration $\geq T$. 
\end{quote}
The point is that the duration may not be integer multiple of $T$, in which case the sampled interval is biased to certain part of a period. 
Besides maxima, minima, and compositions thereof which already satisfy the condition above, a new kind will be introduced.

\subsection{Time-delay embedding and magnitude theory}
\label{subsec:time-series}

\begin{figure*}
\centering
\includegraphics[width=0.99\textwidth]{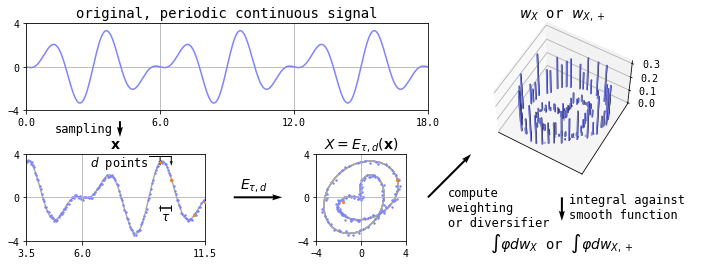}
\caption{Calculation process of weighting and diversifier integrals at a glance. 
}
\label{fig:example-pipeline}
\end{figure*}
Let $\mathbf{s}(t):\bR \rightarrow \bR$ be a continuous function, which will be regarded as a continuous time series. 
Then the \textbf{time-delay embedding} of $\mathbf{s}$ with lag $\tau$ into dimension $d$ is defined as a function 
\begin{align*}
   E_{\tau,d}: \bR &\rightarrow \bR^d \\
    t &\rightarrow (\mathbf{s}(t),\ \mathbf{s}(t+\tau),\ \cdots,\ \mathbf{s}(t+(d-1)\tau).
\end{align*}
In case of finite discrete signal $\hat{\mathbf{s}} = (s_1,\,s_2,\,\cdots,\,s_{L})$ whose sampling frequency $f$ is integer multiple of $\tau^{-1}$, its time-delay embedding is defined in the same way as
\begin{equation*}
    E_{\tau,d}(\hat{\mathbf{s}})
    = \left\{
    \begin{pmatrix}
    s_1 \\ s_{1+f\tau} \\ \vdots \\ s_{1+(d-1)f\tau}
    \end{pmatrix},\, \cdots,\,
    \begin{pmatrix}
    s_{L-{(d-1)f\tau}} \\ s_{L-(d-2)f\tau} \\ \vdots \\ s_{L}
    \end{pmatrix}
    \right\}.
\end{equation*}
Note that in this case, the length shrinks by $(d-1)f\tau$ through embedding.

\highlight{%
Time-delay embedding transforms time series into a point set of a higher dimensional space, 
and any quantity calculated on the point set contains information of the original time series (as already observed by TDA researchers, e.g., in \mbox{\cite{SDB2016}}). 
In particular, weighting and \mbox{\dispersion} of the point set are invariants of given time series. 
Therefore, we suggest the following concepts, but with modifications considering practical computation: we assume given time series is discrete, and we take the integral against a fixed function to obtain scalar values. 
}
\begin{defn}
Let $\hat{\mathbf{s}} = (s_1,\,s_2,\,\cdots,\,s_{L})$ be a finite discrete signal, $E(\hat{\mathbf{s}}) = E_{\tau,d}(\hat{\mathbf{s}})$ its time-delay embedding and $\varphi : \bR^d \rightarrow \bR$ a smooth function. 
The \textbf{weighting integral} (resp. \textbf{diversifier integral}) is the integral
\begin{equation}\label{eq:w-integral}
    \int_{\bR^d} \varphi \, d\wg({E(\hat{\mathbf{s}})})
    \qquad (\text{resp. }
    \int_{\bR^d} \varphi \, d\wgp({E(\hat{\mathbf{s}})})\ )
\end{equation}
of $\varphi$ against the weighting (resp. {\dispersion}) of $E(\hat{\mathbf{s}})$.
\end{defn} 
In spite of different notation, the integrals above are the invariants mentioned in Corollary \ref{cor:1} because every smooth function $\varphi:\bR^d \rightarrow \bR$ belongs to $\cH_X$ when $X = \bR^d$ (see Remark \ref{rmk:R^d}).

\highlight{%
Given a discrete signal $\hat{\mathbf{s}}$ sampled from a continuous periodic signal $\mathbf{s}$, the corresponding integral $\int_{\bR^d} \varphi \, d\wg({E(\hat{\mathbf{s}})})$ is different from $\int_{\bR^d} \varphi \, d\wg({E({\mathbf{s}})})$ in general because the set $E(\hat{\mathbf{s}})$ is different from $E({\mathbf{s}})$. 
However, their Hausdorff distance $d(E(\hat{\mathbf{s}}),E(\mathbf{s}))$ can be made small enough under relatively weak conditions, and so the difference $\int_{\bR^d} \varphi \, d\wg({E(\hat{\mathbf{s}})}) - \int_{\bR^d} \varphi \, d\wg({E({\mathbf{s}})})$ can be made small as well. 

}

\begin{prop}[Stability of weighting integral]\label{prop:main-w}
Suppose that $\mathbf{s}(t)$ is a periodic, univariate, differentiable function with period $T$, and a time-delay embedding $E=E_{\tau,d}$ is fixed. 
Let $\hat{\mathbf{s}}=(s_1,\, s_2,\, \cdots,\, s_{L})$ be a segment sampled from $\mathbf{s}$ with sampling frequency $f$ without error, i.e.
\begin{equation*}\label{eq:sampling-w}
s_k = \mathbf{s} ( t^o + {k}{f^{-1}} )
\qquad(k=1,\, 2,\, \cdots,\, L)
\end{equation*}
where $t^o$ is arbitrary. 
Suppose further that $f$ is an integer multiple of $\tau^{-1}$ and $Lf^{-1}- (d-1)\tau \geq T $, i.e. after time-delay embedding $E(\hat{\mathbf{s}})$ retains at least one period.

Then, for each smooth function $\varphi:\bR^d \rightarrow \bR$, the weighting integral $\int_{\bR^d} \varphi\, d\wg({E(\hat{\mathbf{s}})})$ converges as $f \rightarrow \infty$.
\end{prop}

\highlight{%
The proof of the above proposition is in Appendix \mbox{\ref{sec:proofs}}. 
In the case of m.d. weighting integral we have stronger result as below, whose proof is also in Appendix \mbox{\ref{sec:proofs}}.
}

\begin{prop}[Stability of {\dispersion} integral]\label{prop:main}
Suppose that $\mathbf{s}(t)$ is a periodic, univariate, differentiable function with period $T$, and a time-delay embedding $E=E_{\tau,d}$ is fixed. 
Let $\hat{\mathbf{s}}=(s_1,\, s_2,\, \cdots,\, s_{L})$ be a segment sampled from $\mathbf{s}$ with sampling frequency $f$ and error bound $\delta$, i.e.
    \begin{align}\label{eq:sampling}
    s_k = \mathbf{s} &( t^o + {k}{f^{-1}} ) + e_{k},
    \quad
    |e_{k}| \leq \delta
    \\ \nonumber
    &(k=1,\, 2,\, \cdots,\, L)
    \end{align}
where $t^o$ is arbitrary. 
Suppose further that $f$ is an integer multiple of $\tau^{-1}$ and $Lf^{-1}- (d-1)\tau \geq T $, i.e. after time-delay embedding $E(\hat{\mathbf{s}})$ retains at least one period.

Then, for each smooth function $\varphi:\bR^d \rightarrow \bR$, the {\dispersion} integral $\int_{\bR^d} \varphi\, dw_{E(\hat{\mathbf{s}}),+}$ converges as $f \rightarrow \infty$ and $\delta \rightarrow 0$, with convergence rate $O\left((\delta + \|\mathbf{s}'\|_{\sup} f^{-1} )^{\frac{1}{2}}\right)$. 
\end{prop}

\highlight{%
As seen above, the stability weighting integrals, unlike m.d. weighting integrals, requires the discrete signal $\mathbf{s}$ be sampled \textit{with no errors from a periodic function}. 
Real time series data would hardly meet such a requirement. 
Nevertheless, numerical examples in the following subsections 4.2-4.3 and Section 5 shows weighting integrals work as well as m.d. weighting integrals.
}

\highlight{%
Proofs of above two propositions proceeds roughly as follows. 
By the periodicity of $\mathbf{s}$, $E(\mathbf{s})$ is a (possibly non-simple) closed curve. 
And the set $E(\hat{\mathbf{s}})$ is a finite set close in Hausdorff distance to $E(\mathbf{s})$ provided $Lf^{-1} -(d-1)\tau \geq T$. 
Then Corollary \mbox{\ref{cor:1}} implies that $\int_{\bR^d} \varphi \, d\wg({E(\hat{\mathbf{s}})})$ can be made close enough to $\int_{\bR^d} \varphi \, d\wg({E({\mathbf{s}})})$. 
(For this reason $\wg$ or $\wgp$ cannot be replaced by other measures; e.g., by the uniform measure.)

The approximation of $E(\mathbf{s})$ by $E(\hat{\mathbf{s}}_1)$ and $E(\hat{\mathbf{s}}_2)$ depends only on the sampling frequency and sampling errors. 
A critical point is that the approximation is essentially irrelevant to the length (enough if $Lf^{-1} -(d-1)\tau \geq T$) and phase (the beginning of the segments $\hat{\mathbf{s}}_1$, $\hat{\mathbf{s}}_2$ need not be fixed). 
Due to this, a (m.d.) weighting integral $\int_{\bR^d} \varphi d\wg(E(\hat{\mathbf{s}}))$ of sampled discrete segment $\hat{\mathbf{s}}$ approximates $\int_{\bR^d} \varphi d\wg(E({\mathbf{s}}))$ of the original periodic signal, under little constraints on the length and phase of $\hat{\mathbf{s}}$. 
This property will be more closely observed in Section \mbox{\ref{subsec:toy-examples}}. 
}


For $\varphi$, as a default, we may choose $\varphi_\xi(x) = \cos 2\pi (\xi\cdot x)$ or $\varphi_\xi(x) = \sin 2\pi (\xi\cdot x)$ and consider modulus of Fourier transform $| \widehat{w}(\xi)|  = \left| \int_{\bR^d} e^{-2i\pi (\xi\cdot x)}d\wg({E(\hat{\mathbf{s}})}) \right|$. 
In this case, the integrals $\left\{ \int_{\bR^d} e^{-2i\pi (\xi\cdot x)}d\wg({E(\hat{\mathbf{s}})}) \right\}$ contains sufficient information of $E(\hat{\mathbf{s}})$, in the sense that they can be linearly combined to combinations that approximate any given weighting integral. 
(this is because the collection $\{\varphi_\xi\}$ of functions spans dense,in $\cW =H^{-\frac{d+1}{2}}(\bR^d)$ for example, from Fourier inversion theorem)
Even though it depends on the downstream ML model and the number of $\varphi$'s used whether such linear combinations can be learned, experimental results in Section \ref{sec:experiment} show that this default is sufficient, at least, for the application of interest in this paper.

\subsection{Examples with synthetic data}
\label{subsec:toy-examples}
\begin{figure*}
  \centering
  \includegraphics[width=0.98\textwidth]{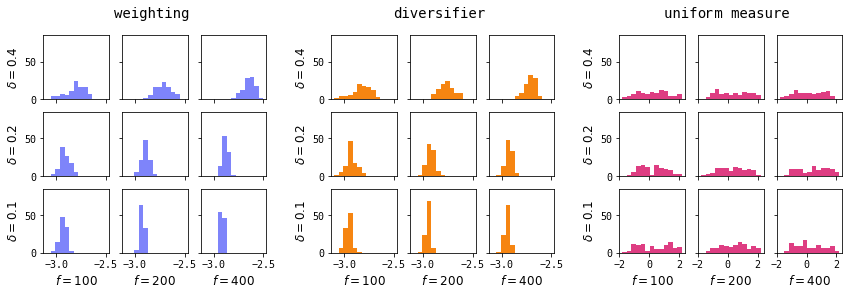}
  \caption{\highlight{Histograms of the values} of weighting integral, diversifier integral, and comparison group, \highlight{for each of 9 pairs of parameters}. Only real parts \highlight{of complex values} are shown for visualization purposes.}
  \label{fig:example-histograms}
\end{figure*}
Let us consider the following periodic continuous signal:
    \begin{equation}
    \mathbf{s} (t)
    = -1.5\sin\pi t + 2\sin \frac{2\pi t}{3}
    \qquad (T=6).
    \end{equation}
We build a dataset by randomly sampling (with noise) $100$ finite sequences of the form of equation \eqref{eq:sampling}, and then apply $E_{\tau, d}$ and compute $\widehat{\wgp({X})}(\xi)$ and $\widehat{\wg(X)}(\xi)$. 
For comparison, the same integral is calculated for the uniform measure as well. 
In this example, $t^o$ and $L$ are randomly chosen where $7 \leq \frac{L}{f} \leq 12$, so the sampling does not respect the period or phase of signals. 
The factors $\tau,\, d$ and $\xi$ are fixed as $\tau=0.5,\, d=2$ and $\xi =(1,-1)$, respectively. 
We let the remaining factors vary as $\delta \in \{0.4,\, 0.2,\, 0.1\}$ and $f \in \{100,\, 200,\, 400\}$. 

The entire process and result are presented in Figure \ref{fig:example-pipeline} and \ref{fig:example-histograms}, respectively. 
The distribution of weighting or diversifier integral shows convergent tendency as $f$ increases and $\delta$ decreases, as proved in Proposition \ref{prop:main}. 
However, such tendency is lost if the measure $\wg(X)$ in integral $\int_{\bR^d} \varphi\, d\wg(X)$ is replaced by uniform measure.

\subsection{Examples with ECG data}
\label{subsec:ecg-examples}

\begin{figure}
\centering
\includegraphics[width=0.8\textwidth]{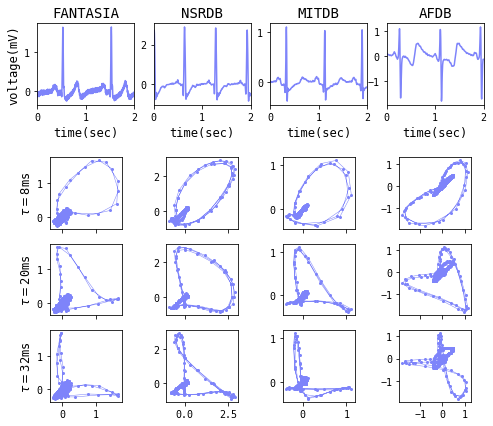} 
\caption{\highlight{Short ECG segments of selected public datasets and their time-delay embeddings with $d=2$ and varying $\tau$.}}
\label{fig:example-ecg-a}
\end{figure}

\begin{figure}
\centering
\includegraphics[width=0.99\textwidth]{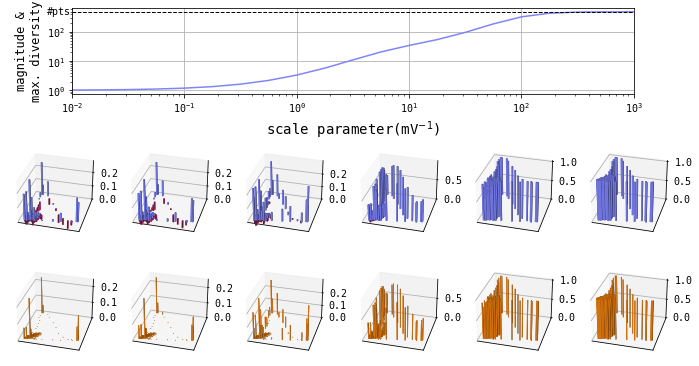}   
\caption{\highlight{Magnitude and relevant concepts computed for the time-delay embedding of an ECG segment under various scale parameters. 
The difference between magnitude and maximum diversity is scarce and not distinguishable in the figure. 
Weighting of the same set is displayed with blue and red bars for positive and negative values. 
Diversifier of the same set is displayed with yellow bars.}}
\label{fig:example-ecg-b}
\end{figure}

\begin{figure}
\centering
\includegraphics[width=0.99\textwidth]{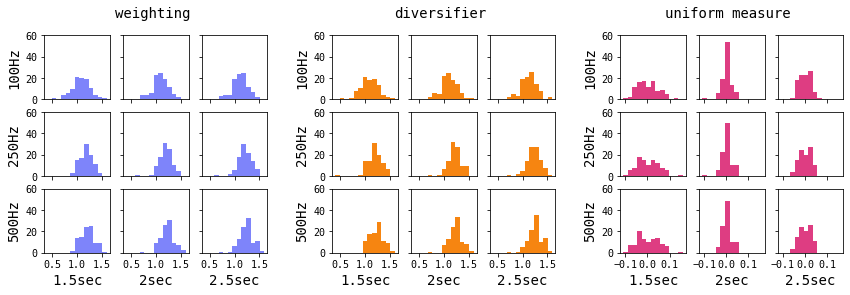}
\caption{\highlight{Histograms of the values of weighting integral, diversifier integral, and comparison group, for each of 9 pairs of parameters.} Weighting(left, blue) and diversifier(middle, yellow) show similar results.}
\label{fig:example-ecg-c}
\end{figure}
Four public datasets from PhysioNet \cite{PhysioNet}: MIT-BIH normal sinus rhythm dataset(NSRDB),\footnote{\url{https://physionet.org/content/nsrdb/1.0.0/}} FANTASIA dataset \cite{FANTASIA}, MIT-BIH arrhythmia dataset(MITDB) \cite{MITDB} and MIT-BIH atrial fibrillation dataset(AFDB) \cite{AFDB}. 
For datasets with multi-lead records, only the first channel is used. 
Voltage values of ECG records of the datasets roughly range over $\pm$ a few millivolts. 
Specifications of each dataset are in Table \ref{table:datasets}. 
Some of the datasets are cleaned before preprocessing because of dataset-specific issues (see Appendix for details). 
Every ECG signal is filtered with 4th order [0.5Hz, 50Hz] bandwidth Butterworth filter for baseline removal. 

\begin{table}
  \caption{A brief specification of datasets.}
  \label{table:datasets}
  \centering
  \begin{tabular}{lllll}
    \hline
    Name & Number of & Health & Sampling & Duration \\
    (abbr.) & subjects & state & frequency & per record \\
    \hline
    FANTASIA & 40 & healthy & 250Hz & $\sim$ 2hr \\
    NSRDB & 18 & healthy & 128Hz & $\sim$ 24hr \\
    MITDB & 47 & mixed & 360Hz & $\sim$ 0.5hr \\
    AFDB & 23 & unhealthy & 250Hz & $\sim$ 10hr \\
    \hline
  \end{tabular}
\end{table}

The examples are calculated as follows. 
We take segments of the first person of each dataset. 
First, the time-delay embedding of the samples are as in Figure \ref{fig:example-ecg-a}. 
The embedding parameters $d=3$ and $\tau = 20$ms are fixed, as this choice seems moderate in that it makes the loss of length $(d-1)f^{-1}\tau$ from time-delay embedding not that large, and at the same time the shape of embedded point set is not much squeezed because of small $d$ and $\tau$. 
Also, the magnitude and maximum diversity of the embedded point set are calculated for different scaling parameters, as commented in Section \ref{sec:background}, in Figure \ref{fig:example-ecg-b}. 
These suggest the default scale $1$ be a moderate choice, in that magnitude lies enough between $1$ and the number of points, which means magnitude `regards' the set neither as a singleton nor as a union of irrelevant points. 
Such scale dependency is also manifest in weighting and {\dispersion}. 
In fact, the similarity between magnitude and maximum diversity implies the similarity between weighting and diversifier, because we have
\begin{align*}
\| \wgp({A}) - \wg(A) \|_{\cW}^2 
& = \magnp(A) + \magn(A) - 2\langle \wgp(A) , \wg(A) \rangle_{\cW} \\
& = \magnp(A) + \magn(A) - 2 \langle Z_A \wg(A), \wgp({A}) \rangle \\
& = \magnp(A) + \magn(A) - 2 \magnp(A) \\
& = \magn(A) - \magnp(A),
\end{align*}
as $Z_A \wg(A) = 1$ on $A$. 

Next, we calculate a weighting and diversifier integral of {linear function} $\varphi(x)=\xi\cdot x$ for segments from the first person of FANTASIA dataset. 
With $\xi = (1, -1, 1) \in \bR^d=\bR^2$, 100 consecutive segments are chosen from the first person, where the timespan of each segment varies over \{1.5s, 2s, 2.5s\} and sampling frequency over \{100Hz, 250Hz, 500Hz \}. 
Like synthetic data, the comparison is made against integrals with weighting replaced by uniform measure as in Figure \ref{fig:example-ecg-c}. 
Again, the distribution of weighting and diversifier integral depends less on the duration of segments and sampling frequency.  
Meanwhile, the uniform measure case also exhibits independence from sampling frequency, but this is not strange. 
More significant is that the distribution changes throughout segments, which means that integral with uniform measure fails to capture information on the shape of the embedded point set. 
Although the variation is less compared to the weighting or {\dispersion} integral, the consequence of such failure will turn out in the next section.

\section{Application to a data analysis problem}
\label{sec:experiment}

The experiments in this section investigate how the suggested (m.d.) weighting integral can be utilized in person identification with the electrocardiogram(ECG). 
\highlight{%
This task is a kind of multi-class classification: For each of the four datasets in Table \mbox{\ref{table:datasets}}, a person identification model should be able to determine to which person a given short ECG segment belongs. 
For example, a model trained with FANTASIA dataset should classify unseen ECG segments from FANTASIA dataset into one of 40 persons.
}


\subsection{Steps of experiment}
Every ECG signal undergoes the following steps during experiments (see Figure \ref{fig:flowchart}).

\paragraph{\textbf{Preprocessing}}{
First, every ECG signal is filtered with 4th order [0.5Hz, 50Hz] bandwidth Butterworth filter for baseline removal, as in subsection \ref{subsec:ecg-examples}. 
Then after resampling to 250Hz, each signal is divided into 2-second-long segments without overlap. 
This blind segmentation does not consider precise beat positions but ensures every segment contains at least one beat by enough length of segment. 
\highlight{Among those segments 1,000 segments are chosen per person randomly, and chosen segments are divided into 60\%:20\%:20\% in temporal order for train, validation, and test, respectively.}
}
\paragraph{\textbf{Feature extraction}}{
As in subsection \ref{subsec:ecg-examples}, each 2-second-long segment is applied to time-delay embedding $E_{\tau,d}$, with $\tau = 20$ms and $d=3$ and no distance scaling (for weighting or {\dispersion}) is applied.\footnote{%
\highlight{Automatized determination of embedding parameters $\tau$ and $d$ is also possible, e.g. by cross-validation. 
Since the experiment focuses more to show the robustness of (m.d.) weighting integral to blind segmentation, the values of $\tau$ and $d$ were fixed manually rather than depending on time-consuming decision processes.
Influence of $\tau$ and $d$ on performance will be discussed later in this section.
}} 
For faster computation of the weighting or {\dispersion}, the embedded point set is downsampled to 200 points (details in Appendix). 
After calculation of weighting or {\dispersion}, the Fourier coefficient is extracted for $512$ $\xi$'s uniformly sampled from $d$-dimensional ball of radius $R$ centered at the origin. 

For better performance, we extract another blind-segmentation-robust feature: we take ``landmark'' $p \in \bR^d$ and compute randomly maximum or minimum distances from each landmark to the embedded set. 
The landmark is sampled from the ball of radius $R'$ centered at the origin. 
That is, for downsampled set $X=\{x_1,\,\cdots,\,x_n\}$ we calculate two types of quantities:
\begin{align}
    \left| \int_{\bR^d} e^{-2\pi i x\cdot \xi} {dw (x)} \right| 
    & \quad ( \| \xi \| \leq R ) \label{eq:ml-fourier}
    \\
    \min \text{ or } \max \{ \|p - x_i\| : x_i \in X\}
    & \quad  (\| p \| \leq R' ),\label{eq:ml-dist}
\end{align}
where $w = w_{X,+}$ or $w=w_X$. 
The values of $R$ and $R'$ were chosen by cross-validation.
Thus, each ECG segment is transformed into a 1,024-dimensional feature vector. 
\highlight{Each feature undergoes normalization with mean and standard deviation.}

To test the utility of suggested integrals, the feature extraction step is modified in two ways in complementary experiments as follows: $w$ in Equation \eqref{eq:ml-fourier} is replaced by uniform measure, or all of 1,024 features are set to the form \eqref{eq:ml-dist}.
}
\paragraph{\textbf{ML models}}{
Four machine learning models are chosen for the experiment: logistic regression(LR), K-nearest neighbors classifier(KNN), support vector machine(SVM), and multi-layer perceptrons(MLP). 
Hyperparameters of the models are determined by cross-validation (details in Appendix \ref{sec:ml-detail}). 
\highlight{The performance of each model is evaluated by accuracy, i.e. the proportion of correctly classified data instances.}
}
\paragraph{\textbf{Hyperparameter selection}}{
\highlight{%
The hyperparameters not clarified so far, $R$, $R'$, and those of ML models, are determined by cross-validation. 
During cross-validation, ML models are trained on train set(60\%) and then tested on validation set(20\%). 
This is done for five times for each combination of hyperparameters, and the combination with the best accuracy is selected. 
The model with the hyperparameters is evaluated through training on train and validation sets(60\%+20\%) and then testing accuracy on test set(20\%).
}
}

\begin{figure}
\centering
\begin{tikzpicture}[node distance=8ex]
\node(ecg){ECG records};

\node(preproc inner)[below of=ecg]{Preprocessing};
\node(baseline)[box,below of=preproc inner,yshift=3ex]{Baseline removal {[0.5Hz, 50Hz]}};
\node(resample)[box,below of=baseline]{Resampling {$f=250$Hz}};
\node(segmentation)[box,below of=resample]{Blind segmentation {2 seconds}};

\node(preproc)[outerbox,fit={(preproc inner) (baseline) (resample) (segmentation)}]{ };

\node(feature inner)[below of=segmentation,yshift=-2ex]{Feature extraction};
\node(embed)[box,below of=segmentation,yshift=-7ex]{time-delay embedding {$d=3$, $\tau=20$ms}};
\node(dnsample)[box,below of=embed]{downsampling {200 pts}};
\node(weighting)[box,below of=dnsample]{weighting/diversifier calculation scale$=1$};
\node(integral)[box,below of=weighting]{integral and min/max distances 1024 features per segment};

\node(feature)[outerbox,fit={(feature inner) (embed) (dnsample) (weighting) (integral)}]{ };

\node(models)[below of=integral]{ML models};

\draw[arrow](ecg) -- (preproc);
\draw[arrow](baseline)--(resample);
\draw[arrow](resample)--(segmentation);

\draw[arrow](preproc)--(feature);

\draw[arrow](embed)--(dnsample);
\draw[arrow](dnsample)--(weighting);
\draw[arrow](weighting)--(integral);

\draw[arrow](feature)--(models);
\end{tikzpicture}
\caption{Experiment steps at a glance.}
\label{fig:flowchart}
\end{figure}

\subsection{Result}
Since randomness occurs in segment selection and choice of $\xi$ and landmarks, as well as in model training, the entire routine is repeated 10 times. 
The accuracy scores are summarized in Table \ref{table:performance-main}, which shows logistic regression and support vector machines achieve the best result among selected ML algorithms, and comparable results to other works in Table \ref{table:sota}. 

Moreover, if weighting or diversifier is replaced by uniform measure or removed, the performance mostly decreases. 
This result can be explained as follows: weighting or diversifier integral is capable of capturing shape information of embedded set, whereas uniform measure cannot. 

Meanwhile, the result with the weighting integral shows similar or better performances than that with the {\dispersion} integral, but this does not seem significant. 
It is rather noteworthy that weighting integral practically works although its stability is not proven, unlike diversifier integral.
In other words, although the usage of weighting is not perfectly supported by theory, considering computation speed and performance, it is worth application in practice. 

Computation time is summarized in Table \ref{table:time-feat} and Table \ref{table:time-model}. 

\begin{table}
  \caption{Mean and standard deviation (in parenthesis) of the accuracy of identification task.}
  \label{table:performance-main}
  \centering
  \begin{tabular}{lcccc}
    \hline
    \multirow{2}{*}{Dataset}  & \multicolumn{4}{c}{Accuracy(\%)} \\
    & LR & KNN & SVM & MLP \\
    \hline
    \multicolumn{5}{c}{weighting case} \\
    \hline
    FANTASIA & \textbf{99.37} (0.04) & 99.29 (0.04) &  99.37 (0.06) & 99.30 (0.05) \\
    NSRDB & {99.30} (0.18) & 98.21 (0.43) &  \textbf{99.44} (0.22) & 98.97 (0.44) \\
    MITDB & 96.72 (0.12) & 96.12 (0.07) &  \textbf{96.82} (0.07) &  96.36 (0.29) \\
    AFDB & 97.45 (0.36) & 96.66 (0.31) &  \textbf{97.89} (0.19) &  97.26 (0.59) \\
    \hline
    \multicolumn{5}{c}{diversifier case} \\
    \hline
    FANTASIA & \textbf{99.36} (0.04) & 99.28 (0.04) &  99.36 (0.06) &  99.29 (0.05) \\
    NSRDB & 99.29 (0.18) & 98.27 (0.39) &  \textbf{99.42} (0.21) &  98.92 (0.35) \\
    MITDB & 96.80 (0.12) & 96.20 (0.05) &  \textbf{96.97} (0.06) &  96.51 (0.23) \\
    AFDB & 97.45 (0.35) & 96.60 (0.26) &  \textbf{97.87} (0.17) &  97.22 (0.47) \\
    \hline
    \multicolumn{5}{c}{uniform measure case} \\
    \hline
    FANTASIA & \textbf{99.28} (0.06) & 98.74 (0.16) &  99.12 (0.10) &  98.98 (0.11) \\
    NSRDB & \textbf{98.95} (0.52) & 96.11 (0.33) &  {98.18} (0.67) &  98.11 (0.97) \\
    MITDB & \textbf{96.54} (0.16) & 94.56 (0.05) &  {95.43} (0.16) &  95.55 (0.42) \\
    AFDB & 97.01 (0.35) & 95.59 (0.35) &  \textbf{97.58} (0.28) &  96.83 (0.26) \\
    \hline
    \multicolumn{5}{c}{default case} \\
    \hline
    FANTASIA & 99.28 (0.06) & 99.11 (0.06) &  \textbf{99.30} (0.05) &  99.16 (0.11) \\
    NSRDB & 99.18 (0.21) & 97.51 (0.44) &  \textbf{99.20} (0.35) &  98.73 (0.46) \\
    MITDB & 96.64 (0.13) & 95.99 (0.11) &  \textbf{97.11} (0.04) &  96.05 (0.58) \\
    AFDB & 95.79 (0.34) & 95.85 (0.32) &  \textbf{97.10} (0.46) &  95.76 (0.82) \\
    \hline
  \end{tabular}
\end{table}

\begin{table}
    \caption{Average computation time for feature extraction (per 1,000 instances).}
    \label{table:time-feat}
    \centering
    \begin{tabular}{rcccc}
    \hline
    & weighting & {\dispersion} & Fourier coefficient & min/max distance \\
    \hline
    time & $\sim$\,5s & $\sim$\,25s & $\sim$\,0.5s & $\sim$\,0.3s \\
    \hline
    \end{tabular}
\end{table}

\begin{table}
    \caption{Average computation time for train and test (after feature extraction).}
    \label{table:time-model}
    \centering
    \begin{tabular}{lcccc}
    \hline
    \multicolumn{5}{c}{train} \\
    \hline
    Dataset & LR & KNN & SVM & MLP \\
    \hline
    FANTASIA & 110s & $<$1s & 7s & 40s \\
    NSRDB & 41s & $<$1s & 3s & 16s \\
    MITDB & 352s & $<$1s& 16s & 144s\\
    AFDB & 99s & $<$1s & 9s& 40s\\
    \hline
    \multicolumn{5}{c}{test} \\
    \hline
    Dataset & LR & KNN & SVM & MLP \\
    \hline
    FANTASIA &  $<$1s & 2s & 11s & $<$1s \\
    NSRDB & $<$1s & $<$1s & 3s & $<$1s \\
    MITDB & $<$1s & 2s & 39s & $<$1s\\
    AFDB & $<$1s & $<$1s & 9s& $<$1s\\
    \hline
    \end{tabular}
\end{table}

\begin{table*}
  \caption{Comparison with other studies.}
  \label{table:sota}
  \centering
  \begin{tabular}{llllll}
    \hline
    \multirow{2}{*}{Study} &  \multirow{2}{*}{Method} & \multicolumn{4}{c}{Accuracy(\%)} \\
    && FANTASIA & NSRDB & MITDB & AFDB \\
    \hline
    \hline
    \multirow{2}{*}{Proposed}  & {time-delay embedding} & 
    \multirow{2}{*}{99.37} & \multirow{2}{*}{99.44}& \multirow{2}{*}{96.97}& \multirow{2}{*}{97.89} \\ 
    & + weighting + SVM&&&& \\
    \hline 
    \cite{KKP2020} & ECG coupling image + CNN & - &99.2&-&- \\
    \hline
    \cite{LPWL2020} & {cascaded CNN}& 99.9 & 96.1 &- &90.9 \\
    \hline
    \cite{Bento+2019} & spectrograms + DenseNet & 99.79 &-&-&- \\
    \hline
    \cite{CSH2019} & multi-scale residual & - & 97.17 & 95.99 & - \\
    \hline
    \cite{ZZZ2017} & {wavelet transform + MCNN}& 97.2 & 95.1 & 91.1 & 93.9 \\
    \hline
  \end{tabular}
\end{table*}

\subsection{Selection of hyperparameters}
It is difficult to theoretically trace the correlation between $d$, $\tau$, and the overall performance because the performance depends on the domain data and specific tasks, and partly because there are no restrictions on the hyperparameters $d$, $\tau$ in the core proposition(Proposition \ref{prop:main}). 
Therefore, the dependency of the result on the hyperparameters should be observed empirically. 
In this paper, the choice $d=3$ and $\tau=20$ms was based on qualitative observation of data, but the performance is fair enough compared to other choices as in Table \ref{table:performance-ablation}. 
In general, the hyperparameters could also be chosen by cross-validation to minimize human intervention inside the data analysis pipeline.

As for the choice between weighting and {\dispersion}, the result in Table \ref{table:performance-main} shows only marginal differences. 
Then, weighting is recommended because of the lower computational cost. 

\begin{table}
  \caption{Mean and standard deviation (in parenthesis) of the accuracy of identification task under different embedding parameters($d$ and $\tau$) and weighting.}
  \label{table:performance-ablation}
  \centering
  \begin{tabular}{llccccc}
    \hline
    \multirow{2}{*}{Dataset}  & \multirow{2}{*}{Model}  & \multicolumn{5}{c}{$d=3$} \\
    & & $\tau=8$ms & $\tau=20$ms & $\tau=32$ms & $\tau=44$ms & $\tau=56$ms \\
    \hline
    \multirow{2}{*}{NSRDB} 
    & LR  &98.96 (0.25)& 99.30 (0.18) &98.85 (0.50)& 99.03 (0.40)&98.97 (0.26) \\
    & SVM  &99.11 (0.27)& 99.44 (0.22) &99.17 (0.34)&98.96 (0.58)&98.73 (0.39) \\
    \multirow{2}{*}{AFDB} 
    & LR  &96.61 (0.59)& 97.45 (0.36) &97.16 (0.43)&96.56 (0.44)&95.36 (0.74) \\
    & SVM  &97.32 (0.32)& 97.89 (0.19) &97.62 (0.23)&97.14 (0.22)&96.66 (0.42) \\
    \hline
    \multirow{2}{*}{Dataset}  & \multirow{2}{*}{Model}  & \multicolumn{5}{c}{$\tau=20$ms} \\
    && $d=2$ & $d=3$ & $d=4$ & $d=5$ & $d=6$ \\
    \hline
    \multirow{2}{*}{NSRDB} 
    & LR  & 98.79 (0.25)& 99.30 (0.18) &99.38 (0.18)& 99.46 (0.15)&99.46 (0.09) \\
    & SVM  &99.12 (0.19)& 99.44 (0.22) &99.54 (0.14)&99.52 (0.19)&99.49 (0.15) \\
    \multirow{2}{*}{AFDB} 
    & LR  &95.05 (0.56) & 97.45 (0.36) &97.88 (0.22)&98.12 (0.21)&98.08 (0.39) \\
    & SVM  &96.51 (0.47)& 97.89 (0.19) &98.07 (0.15)&98.17 (0.19)&98.12 (0.24) \\
    \hline
  \end{tabular}
\end{table}

\subsection{Comparison with other studies}

It is not trivial to compare ECG individual identification performances among different studies because evaluation methods are not standardized. 
For example, sometimes longer segments are allowed (up to 9 heartbeats) during inference \cite{HYWLY2021,KP2020}, or the amount of data is restricted during training \cite{SK2017,ICTPK2020}. 
Other works with relatively similar standards are listed in Table \ref{table:sota}. 
Still, the segmentization method also matters because heartbeat-based segmentization as in \cite{KKP2020,LPWL2020,CSH2019} provides different train and test sets than blind segmentation as in \cite{Bento+2019,ZZZ2017}. 
Or simply the amount of data for training is smaller in the case of \cite{ZZZ2017} or \cite{CSH2019}. 
Nevertheless, the suggested invariant in this paper has advantages in the simplicity of the model while keeping moderate performance compared to other papers. 
This is also meaningful in that issues from blind segmentation (e.g., the position of the heartbeat is not constant, the number of heartbeats in each segment varies) are solved without recourse to DNN models but with geometric ideas. 

\subsection{Limitations}

For real application in ECG biometrics, the following should be further considered with the proposed invariant. 
First, every person's ECG signal changes over time, and long-term usage should be considered also. 
To this end, several works \cite{SSS2021,ADI2021,MTV2023} evaluate algorithms in multi-session datasets, in which several ECG signals over long periods are collected from each person. 
Second, ECG collection with more causal devices is more applicable for person identification. 
Unlike the datasets used in this paper, there are ``off-the-person'' datasets, which are collected by less invasive devices and used in works like \cite{CC2019,ICTPK2020,SSS2021}. 
The works \cite{MTV2023,ICTPK2020} compare performance with those more realistic datasets and show that the identification model may perform less on multi-session or off-the-person datasets. 
These settings are not covered in this paper and therefore should be included in the scope of further research. 

\highlight{%
The current experiment focused more on the verification of the robustness of (m.d.) weighting integrals to blind segmentation (and indirectly, the continuity of (m.d.) weighting). 
Before studying how to achieve state-of-the-art level performance with (m.d.) weighting integrals, there are more questions to be answered. 
How the parameters $d$ and $\tau$ of time-delay embedding should be determined, without depending on time-consuming cross-validation process? 
How the (m.d.) weighting of a time series (or a point cloud in general) can be transformed into an element of $\bR^N$ for fixed $N$, say for data analysis purpose? 
As used in this article, integral against a test function $\varphi$ is also possible, but this added \textit{a function} to the hyperparameters to be determined. 
The time complexity of weighting and m.d. weighting also matters: the former requires solving linear equation, and the latter requires do convex optimization. 
For broader application, efficient computation technique should be studied.
}

\section{Conclusion}
\label{sec:conclusion}

This article has \highlight{first proposed novel theoretical facts on magnitude and weighting and similar facts on maximum diversity. 
A new concept of maximum diversity weighting was suggested as a measure that contains strictly more information than maximum diversity, similarly as weighting is to magnitude. 
The continuity of maximum diversity was able to be transferred to the continuity of m.d. weighting, similarly as the continuity of magnitude (which is not always the case, though) to weighting. 

Then, these theoretical results was associated with an application to time series analysis of periodic time series.
Periodic time series is first transformed into (a subset of) a closed curve, and the (m.d.) weighting of the transformed set produced invariants by integration against test functions. 
In the case of weighting integral, the condition for stability is not exactly satisfied with real data. 
But in the experiments weighting integrals performed as well as m.d. weighting integrals, whose stability is theoretically supported.

In other words, this article illuminated the utility of (m.d.) weighting as an indicator of ``shape'' of given set, in the sense that it satisfies (semi-)continuity with respect to Hausdorff distance. 
}
To the best of the author's knowledge, there are no other notions of measure or distribution, even for a finite subset of $\bR^d$, which depends on the shape rather than the ``local density'' of a metric space. 
\highlight{%
Such a property can perhaps be utilized in certain application domains, in processing point cloud data for example. 
This should be treated in subsequent studies, in addition to the theoretical study on magnitude and weighting, 
}


\section*{Conflict of Interest}
The authors declare no conflict of interest.

\section*{Use of AI tools declaration}
The authors declare they have not used Artificial Intelligence (AI) tools in the creation of this article.

\appendix

\section{\highlight{Proof of stability propositions}}
\label{sec:proofs}

\begin{proof}[Proof of Proposition \ref{prop:main-w}]
Let $\hat{A} = E(\hat{\mathbf{s}})$ and $A = E(\mathbf{s})$. 
By the assumption $Lf^{-1}-(d-1)\tau \geq T$, for each $t \in \bR$, there exists $k \in \{ 1,2,\cdots,L-(d-1)f\tau\}$ such that $|t - (t^o + kf^{-1})| < \frac{1}{2f}$. 
Then $d_H(A,\hat{A}) < \sup \{ \| E(\mathbf{s})(t_1) - E(\mathbf{s})(t_2) \| : |t_1 - t_2 | < \frac{1}{2f} \}$. 
Since $\hat{A} \subset A$, Proposition \ref{prop:mm-conti}(1) and Lemma \ref{lem:wm-ineq} together imply that $\int_{\bR^d} \varphi d \wg (E(\hat{\mathbf{s}}))$ converges as $f \rightarrow \infty$ to $\int_{\bR^d} \varphi d \wg (E({\mathbf{s}}))$.
\end{proof}

\begin{proof}[Proof of Proposition \ref{prop:main}]
Let $\hat{A}$ be the time-delay embedding of a segment $\hat{\mathbf{s}}$ satisfying the hypothesis with relevant parameters $f$ and $\delta$. 
Likewise, let $A$ be the time-delay embedding of the original signal $\mathbf{s}$.

We need to estimate both terms inside $\max$ of $d_H(A,\hat{A}) = \max\{ \sup_{x\in A} d(x,\hat{A}), \sup_{y \in \hat{A}} d(A, y) \}$. 
By construction, every element $y \in \hat{A}$ is of the form
\begin{equation}\label{eq:elt-of-discrete-embedding}
    y = 
    \begin{pmatrix}
    \mathbf{s}(t^o + kf^{-1}) \\ 
    \mathbf{s}(t^o + kf^{-1}+\tau) \\
    \vdots \\
    \mathbf{s}(t^0 + kf^{-1}+(d-1)\tau )
    \end{pmatrix}
    +
    \begin{pmatrix}
    e_k \\
    e_{k+\tau f} \\
    \vdots \\
    e_{k+(d-1)\tau f}
    \end{pmatrix},
\end{equation}
and the first term of right-hand side is in $A$ and the second term has norm $\leq \sqrt{d}\delta$.
Hence, we have $\sup_{y\in \hat{A}} d(A,y) \leq \sqrt{d}\delta$.
On the other hand, for every element $x = (\mathbf{s}(c),\,\mathbf{s}(c+\tau),\,\cdots,\,\mathbf{s}(c+(d-1)\tau)) \in A$ we can choose element $y$ of the form \eqref{eq:elt-of-discrete-embedding} such that $|c - (t^0 + kf^{-1})| \leq \frac{1}{2}f^{-1}$. 
Then we have $d(x,y) \leq \frac{1}{2}\sqrt{d}\|\mathbf{s}'\|_{\sup} f^{-1}$, which in turn implies $\sup_{x\in A} d(x,\hat{A}) \leq \sqrt{d}\|\mathbf{s}'\|_{\sup} f^{-1} +\sqrt{d}\delta$.
Combining the results, we have
    \begin{equation*}
    d_H(A,\hat{A}) \leq \sqrt{d} (\delta + \| \mathbf{s}' \|_{\sup} f^{-1}),
    \end{equation*}
and then Proposition \ref{prop:mw-diff-estimate} gives
    \begin{equation*}
    \begin{aligned}
    \left| \int_{\bR^d} \varphi \, dw_{A,+} - \int_{\bR^d} \varphi \, dw_{\hat{A},+} \right|
    &  = \langle Z^{-1}\varphi , w_{A,+} - w_{\hat{A},+} \rangle_{\cW} \\
    & \leq \| Z^{-1}\varphi \| _{\cW}\, \| w_{A,+} - w_{\hat{A},+} \| _{\cW} \\
    & \leq 4\sqrt{2} \, \| Z^{-1}\varphi \|_{\cW} \,  |A \cup \hat{A}|_+ \, d ^{\frac{1}{4}} (\delta + \| \mathbf{s}' \|_{\mathrm{sup}} f^{-1} )^{\frac{1}{2}}.
    \end{aligned}
    \end{equation*}
Let $B$ be the closure of radius $1$ neighborhood of $A$. 
For sufficiently small $\delta$ and large $f$ so that $d_H(A,\hat{A}) \leq \delta+\|\mathbf{s}' \|_{\sup} f^{-1 } < 1$, we have $|A \cup \hat{A}|_+ \leq |B|_+ < \infty$. 
The proof is complete.
\end{proof}

\section{Details of data analysis experiment} \label{sec:ml-detail}

This section supplements details for implementing the experiment in Section \ref{sec:experiment} of the main article. 
Codes are available at the GitHub repository (\url{https://github.com/sinwall/magnitude-invariant-ecg}).

\subsection*{Hardward specifications}

The hardware used is equipped with Intel(R) Core(TM) i7-10700K CPU and 64.0GB RAM and without GPU, and as software Python 3.8.6 is used. 
Omitted details on code implementation, hyperparameters, etc. can be found in the Appendix. 

\subsection*{Dataset-specific issues}

Several signals of NSRDB start and end with some minutes of artifacts without any beat annotation, so those are removed in our experiments for proper evaluation. 
FANTASIA dataset contains NaN values, sometimes too long to interpolate. 
Therefore the segments with NaNs are removed after blind segmentation.

\subsection*{Implementaion of feature extraction}
As mentioned in Section \ref{sec:experiment} Python 3.8.6 is used to conduct experiments. 
Computation of weighting is just solving linear equation $Z_X w_X = \mathbf{1}$, so this is implemented by \texttt{numpy.linalg.lstsq} of Numpy \cite{python.numpy}, and just-in-time compilation of Numba \cite{python.numba} for faster computation. 
On the other hand, computation of {\dispersion} requires convex optimization, so it is implemented with CVXPY \cite{python.cvxpy,python.cvxpy.re}, more precisely by implementing equivalent problem mentioned in the proof of Lemma \ref{lem:Zw+=1+}.

Both computations take longer time as the embedded point set becomes larger. 
Therefore the point set is downsampled after time-delay embedding and before computation of weighting or {\dispersion}. 
Though not in an optimal way, but enough for application, the Algorithm \ref{alg:reduce} drops points that have small distances with succeeding or proceeding points.

The radii $R$ and $R'$ for balls from which $\xi$ and $p$ (of Equations \eqref{eq:ml-fourier} and \eqref{eq:ml-dist}) are uniformly sampled, respectively, are choosen among $\{0.25,\,0.5,\,1,\,2,\,4\}$, considering the values of ECG signal ranges $\pm$ few milivolts, as in Figure \ref{fig:example-ecg-a}. 
To save time, the choice for $R$ and $R'$ is based on performances of KNN model only, with downsampling to $100$ points, and only $256$ feature per segment. 

\begin{algorithm}
    \caption{Downsampling points from set.}
    \label{alg:reduce}
    \begin{algorithmic}[1]
    \Procedure{Downsample}{$X, k$}
        \If{$k == 0$}
            \State \Return $X$
        \EndIf
        \State $x_0,\,x_1,\,x_2,\,\cdots \leftarrow$ elements of $X$ in temporal order 
        \State $distances \leftarrow [ d(x_0, x_1),\, d(x_2, x_3),\, d(x_4, x_5),\cdots\ ]$
        \State $d(x_{i_1-1},x_{i_1}),\, d(x_{i_2-1},x_{i_2}),\ \cdots\, \leftarrow$ $\lceil k/2 \rceil$ least entries of $distances$
        \State $Y \leftarrow $ remove $x_{i_1},\, x_{i_2},\, \cdots$ from $X$
        \State \Return \textsc{Downsample}($Y$, $\lfloor k/2 \rfloor$)
    \EndProcedure
    \end{algorithmic}
\end{algorithm}
\subsection*{Specification of ML models}
Machine learning algorithms are implemented by Scikit-learn \cite{python.sklearn}. 
The APIs and hyperparameters are as follows.
\begin{itemize}
\item For LR (\texttt{.linear\_model.LogisticRegression}) \texttt{C} is choosen among $\{10^{-2}, 10^{-1}, 10^{0}, 10^{1}, 10^{2}\}$.
\item For KNN (\texttt{.neighbors.KNeighborsClassifier}) \texttt{n\_neighbors} and \texttt{weights} are choosen among $\{1,\,2,\,5,\,10\}$ and \{`unifrom',\,`distance'\}, respectively.
\item For SVM (\texttt{.svm.SVC}) \text{C} and \texttt{gamma} are choosen from $\{10^{-1}, 10^{0}, 10^{1}\}$ and $\{10^{-4}, 10^{-3}, 10^{-2}\}$, respectively.
\item For MLP (\texttt{.neural\_network.MLPClassifier}) \texttt{hidden\_layers} are choosen from $\{(64,\,),\,(128,\,),\,(256,\,),\,(512,\,)\}$.
\end{itemize}
Every other parameter is set to default of Scikit-learn, except \texttt{random\_state} which is controlled in another way for reproducibility.

Cross-validation for model parameters is done after values of $R$ and $R'$ are determined ahead. 
To save time, downsampling is done to $100$ points as above but extracted features are $1,024$ per segment.

\bibliographystyle{abbrv}
\bibliography{references}

@article{Meckes2020,
   author = {Mark W. Meckes},
   doi = {10.1112/mtk.12024},
   issn = {0025-5793},
   issue = {2},
   journal = {Mathematika},
   month = {4},
   pages = {343-355},
   title = {ON THE MAGNITUDE AND INTRINSIC VOLUMES OF A CONVEX BODY IN EUCLIDEAN SPACE},
   volume = {66},
   year = {2020}
}

@article{GimperleinGoffeng2021,
   author = {Heiko Gimperlein and Magnus Goffeng},
   doi = {10.1353/ajm.2021.0023},
   issn = {1080-6377},
   issue = {3},
   journal = {American Journal of Mathematics},
   pages = {939-967},
   title = {On the magnitude function of domains in Euclidean space},
   volume = {143},
   year = {2021}
}

@article{LeinsterCobbold2012,
   author = {Tom Leinster and Christina A. Cobbold},
   doi = {10.1890/10-2402.1},
   issn = {0012-9658},
   issue = {3},
   journal = {Ecology},
   month = {3},
   pages = {477-489},
   title = {Measuring diversity: the importance of species similarity},
   volume = {93},
   year = {2012}
}

@article{Hjorth+1998,
   author = {Poul Hjorth and Petr Lisonĕk and Steen Markvorsen and Carsten Thomassen},
   doi = {10.1016/S0024-3795(97)00242-5},
   issn = {00243795},
   issue = {1-3},
   journal = {Linear Algebra and its Applications},
   month = {2},
   pages = {255-273},
   title = {Finite metric spaces of strictly negative type},
   volume = {270},
   year = {1998}
}

@inproceedings{Andreeva+2023,
   author = {Rayna Andreeva and Katharina Limbeck and Bastian Rieck and Rik Sarkar},
   edition = {TAG-ML, ICML},
   booktitle = {Proceedings of 2nd Annual Workshop on Topology, Algebra, and Geometry in Machine Learning (TAG-ML)},
   month = {7},
   note = {The Fortieth International Conference on Machine Learning, ICML 2023 ; Conference date: 23-07-2023 Through 29-07-2023},
   pages = {242-253},
   publisher = {Journal of Machine Learning Research: Workshop and Conference Proceedings},
   title = {Metric space magnitude and generalisation in neural networks},
   volume = {221},
   url = {https://icml.cc/},
   year = {2023}
}

@article{HepworthWillerton2017,
   author = {Richard Hepworth and Simon Willerton},
   doi = {10.4310/HHA.2017.v19.n2.a3},
   issn = {15320073},
   issue = {2},
   journal = {Homology, Homotopy and Applications},
   pages = {31-60},
   title = {Categorifying the magnitude of a graph},
   volume = {19},
   year = {2017}
}

@article{KP2020,
   author = {Beom-Hun Kim and Jae-Young Pyun},
   doi = {10.3390/s20113069},
   issn = {1424-8220},
   issue = {11},
   journal = {Sensors},
   title = {ECG Identification For Personal Authentication Using LSTM-Based Deep Recurrent Neural Networks},
   volume = {20},
   url = {https://www.mdpi.com/1424-8220/20/11/3069},
   year = {2020},
}

@article{Bento+2019,
   author = {Nuno Bento and David Belo and Hugo Gamboa},
   doi = {10.18178/ijmlc.2020.10.2.929},
   issue = {2},
   journal = {Int. J. Mach. Learn.},
   pages = {259-264},
   title = {ECG biometrics using spectrograms and deep neural networks},
   volume = {10},
   year = {2020},
}

@article{CSH2019,
   author = {Yifan Chu and Haibin Shen and Kejie Huang},
   doi = {10.1109/ACCESS.2019.2912519},
   issn = {2169-3536},
   journal = {IEEE Access},
   pages = {51598-51607},
   title = {ECG Authentication Method Based on Parallel Multi-Scale One-Dimensional Residual Network With Center and Margin Loss},
   volume = {7},
   year = {2019},
}

@article{python.numpy,
   author = {Charles R. Harris and K. Jarrod Millman and Stéfan J. van der Walt and Ralf Gommers and Pauli Virtanen and David Cournapeau and Eric Wieser and Julian Taylor and Sebastian Berg and Nathaniel J. Smith and Robert Kern and Matti Picus and Stephan Hoyer and Marten H. van Kerkwijk and Matthew Brett and Allan Haldane and Jaime Fernández del Río and Mark Wiebe and Pearu Peterson and Pierre Gérard-Marchant and Kevin Sheppard and Tyler Reddy and Warren Weckesser and Hameer Abbasi and Christoph Gohlke and Travis E. Oliphant},
   doi = {10.1038/s41586-020-2649-2},
   issn = {0028-0836},
   issue = {7825},
   journal = {Nature},
   month = {9},
   pages = {357-362},
   title = {Array programming with NumPy},
   volume = {585},
   year = {2020},
}

@article{KKP2020,
   author = {Jin Su Kim and Sung Hyuck Kim and Sung Bum Pan},
   doi = {10.1007/s12652-019-01401-3},
   issn = {1868-5145},
   issue = {5},
   journal = {Journal of Ambient Intelligence and Humanized Computing},
   pages = {1923-1932},
   title = {Personal recognition using convolutional neural network with ECG coupling image},
   volume = {11},
   url = {https://doi.org/10.1007/s12652-019-01401-3},
   year = {2020},
}

@article{ADI2021,
   author = {Dalal A AlDuwaile and Md Saiful Islam},
   doi = {10.3390/e23060733},
   issn = {1099-4300},
   issue = {6},
   journal = {Entropy},
   title = {Using Convolutional Neural Network and a Single Heartbeat for ECG Biometric Recognition},
   volume = {23},
   url = {https://www.mdpi.com/1099-4300/23/6/733},
   year = {2021},
}

@article{ICTPK2020,
   author = {Mohit Ingale and Renato Cordeiro and Siddartha Thentu and Younghee Park and Nima Karimian},
   doi = {10.1109/ACCESS.2020.3004464},
   journal = {IEEE Access},
   pages = {117853-117866},
   title = {ECG Biometric Authentication: A Comparative Analysis},
   volume = {8},
   year = {2020},
}

@article{MTV2023,
   author = {Pietro Melzi and Ruben Tolosana and Ruben Vera-Rodriguez},
   doi = {10.1109/ACCESS.2023.3244651},
   journal = {IEEE Access},
   pages = {15555-15566},
   title = {ECG Biometric Recognition: Review, System Proposal, and Benchmark Evaluation},
   volume = {11},
   year = {2023},
}

@article{SSS2021,
   author = {Ranjeet Srivastva and Ashutosh Singh and Yogendra Narain Singh},
   doi = {10.1016/j.ins.2021.01.001},
   issn = {0020-0255},
   journal = {Information Sciences},
   pages = {208-228},
   title = {PlexNet: A fast and robust ECG biometric system for human recognition},
   volume = {558},
   url = {https://www.sciencedirect.com/science/article/pii/S0020025521000025},
   year = {2021},
}

@inproceedings{CC2019,
   author = {Iulian B. Ciocoiu and Nicolae Cleju},
   doi = {10.1109/ISSCS.2019.8801783},
   isbn = {978-1-7281-3896-1},
   booktitle = {2019 International Symposium on Signals, Circuits and Systems (ISSCS)},
   month = {7},
   pages = {1-4},
   publisher = {IEEE},
   title = {Off-the-person ECG Biometrics Using Convolutional Neural Networks},
   url = {https://ieeexplore.ieee.org/document/8801783/},
   year = {2019},
}

@inbook{LeinsterMeckes2017,
   author = {Tom Leinster and Mark W Meckes},
   city = {Warsaw, Poland},
   doi = {doi:10.1515/9783110550832-005},
   isbn = {9783110550832},
   booktitle = {Measure Theory in Non-Smooth Spaces},
   pages = {156-193},
   publisher = {De Gruyter Open Poland},
   title = {The magnitude of a metric space: from category theory to geometric measure theory},
   url = {https://doi.org/10.1515/9783110550832-005},
   year = {2017}
}

@article{python.cvxpy,
   author = {Steven Diamond and Stephen Boyd},
   issue = {83},
   journal = {Journal of Machine Learning Research},
   pages = {1-5},
   title = {CVXPY: A Python-embedded modeling language for convex optimization},
   volume = {17},
   year = {2016},
}

@article{LPWL2020,
   author = {Yazhao Li and Yanwei Pang and Kongqiao Wang and Xuelong Li},
   doi = {10.1016/j.neucom.2020.01.019},
   issn = {0925-2312},
   journal = {Neurocomputing},
   pages = {83-95},
   title = {Toward improving ECG biometric identification using cascaded convolutional neural networks},
   volume = {391},
   url = {https://www.sciencedirect.com/science/article/pii/S0925231220300485},
   year = {2020},
}

@article{python.cvxpy.re,
   author = {Akshay Agrawal and Robin Verschueren and Steven Diamond and Stephen Boyd},
   doi = {10.1080/23307706.2017.1397554},
   issue = {1},
   journal = {Journal of Control and Decision},
   pages = {42-60},
   title = {A rewriting system for convex optimization problems},
   volume = {5},
   year = {2018},
}

@article{FANTASIA,
   author = {N Iyengar and C K Peng and R Morin and A L Goldberger and L A Lipsitz},
   doi = {10.1152/ajpregu.1996.271.4.R1078},
   issue = {4},
   journal = {American Journal of Physiology-Regulatory, Integrative and Comparative Physiology},
   note = {PMID: 8898003},
   pages = {R1078-R1084},
   title = {Age-related alterations in the fractal scaling of cardiac interbeat interval dynamics},
   volume = {271},
   url = {https://doi.org/10.1152/ajpregu.1996.271.4.R1078},
   year = {1996},
}

@article{AFDB,
   author = {G. Moody and R. Mark},
   doi = {doi.org/10.13026/C2MW2D},
   journal = {Computers in Cardiology},
   pages = {227-230},
   title = {A new method for detecting atrial fibrillation using R-R intervals},
   volume = {10},
   year = {1983},
}

@article{Willerton2015,
   author = {Simon Willerton},
   doi = {10.1142/S0218195915500120},
   issue = {03},
   journal = {International Journal of Computational Geometry \& Applications},
   pages = {207-225},
   title = {Spread: A Measure of the Size of Metric Spaces},
   volume = {25},
   url = {https://doi.org/10.1142/S0218195915500120},
   year = {2015},
}

@article{ZZZ2017,
   author = {Qingxue Zhang and Dian Zhou and Xuan Zeng},
   doi = {10.1109/ACCESS.2017.2707460},
   journal = {IEEE Access},
   pages = {11805-11816},
   title = {HeartID: A Multiresolution Convolutional Neural Network for ECG-Based Biometric Human Identification in Smart Health Applications},
   volume = {5},
   year = {2017},
}

@inproceedings{SK2017,
   author = {Ronald Salloum and C.-C. Jay Kuo},
   doi = {10.1109/ICASSP.2017.7952519},
   isbn = {978-1-5090-4117-6},
   booktitle = {2017 IEEE International Conference on Acoustics, Speech and Signal Processing (ICASSP)},
   month = {3},
   pages = {2062-2066},
   publisher = {IEEE},
   title = {ECG-based biometrics using recurrent neural networks},
   url = {http://ieeexplore.ieee.org/document/7952519/},
   year = {2017},
}

@article{HYWLY2021,
   author = {Yu-Wen Huang and Gong-Ping Yang and Kui-Kui Wang and Hai-Ying Liu and Yi-Long Yin},
   doi = {10.1007/s11390-021-1033-5},
   issn = {1860-4749},
   issue = {3},
   journal = {Journal of Computer Science and Technology},
   pages = {617-632},
   title = {Multi-Scale Deep Cascade Bi-Forest for Electrocardiogram Biometric Recognition},
   volume = {36},
   url = {https://doi.org/10.1007/s11390-021-1033-5},
   year = {2021},
}

@inproceedings{SDB2016,
   author = {Lee M Seversky and Shelby Davis and Matthew Berger},
   doi = {10.1109/CVPRW.2016.131},
   issn = {2160-7516},
   booktitle = {2016 IEEE Conference on Computer Vision and Pattern Recognition Workshops (CVPRW)},
   month = {6},
   pages = {1014-1022},
   title = {On Time-Series Topological Data Analysis: New Data and Opportunities},
   year = {2016},
}

@article{python.sklearn,
   author = {Fabian Pedregosa and Gaël Varoquaux and Alexandre Gramfort and Vincent Michel and Bertrand Thirion and Olivier Grisel and Mathieu Blondel and Peter Prettenhofer and Ron Weiss and Vincent Dubourg and Jake Vanderplas and Alexandre Passos and David Cournapeau and Matthieu Brucher and Matthieu Perrot and Édouard Duchesnay},
   issue = {85},
   journal = {Journal of Machine Learning Research},
   pages = {2825-2830},
   title = {Scikit-learn: Machine Learning in Python},
   volume = {12},
   url = {http://jmlr.org/papers/v12/pedregosa11a.html},
   year = {2011},
}

@inproceedings{python.numba,
   author = {Siu Kwan Lam and Antoine Pitrou and Stanley Seibert},
   city = {New York, NY, USA},
   doi = {10.1145/2833157.2833162},
   isbn = {9781450340052},
   booktitle = {Proceedings of the Second Workshop on the LLVM Compiler Infrastructure in HPC},
   publisher = {Association for Computing Machinery},
   title = {Numba: a LLVM-based Python JIT compiler},
   url = {https://doi.org/10.1145/2833157.2833162},
   year = {2015},
}

@article{Leinster2008,
   author = {Tom Leinster},
   doi = {10.4171/dm/240},
   issn = {1431-0635},
   journal = {Documenta Mathematica},
   pages = {21-49},
   title = {The Euler characteristic of a category},
   volume = {13},
   year = {2008},
}

@misc{NSRDB,
   author = {George Moody},
   month = {8},
   title = {MIT-BIH Normal Sinus Rhythm Database},
   year = {1999},
}

@article{Meckes2015,
   author = {Mark W Meckes},
   doi = {10.1007/s11118-014-9444-3},
   issn = {1572-929X},
   issue = {2},
   journal = {Potential Analysis},
   pages = {549-572},
   title = {Magnitude, Diversity, Capacities, and Dimensions of Metric Spaces},
   volume = {42},
   url = {https://doi.org/10.1007/s11118-014-9444-3},
   year = {2015},
}

@article{BarceloCarbery2018,
   author = {Juan Antonio Barceló and Anthony Carbery},
   doi = {10.1353/ajm.2018.0011},
   issn = {1080-6377},
   issue = {2},
   journal = {American Journal of Mathematics},
   month = {4},
   pages = {449-494},
   title = {On the magnitudes of compact sets in Euclidean spaces},
   volume = {140},
   url = {https://muse.jhu.edu/article/688522},
   year = {2018},
}

@article{Meckes2013,
   author = {Mark W Meckes},
   doi = {10.1007/s11117-012-0202-8},
   issn = {1572-9281},
   issue = {3},
   journal = {Positivity},
   pages = {733-757},
   title = {Positive definite metric spaces},
   volume = {17},
   url = {https://doi.org/10.1007/s11117-012-0202-8},
   year = {2013},
}

@article{PhysioNet,
   author = {Ary L Goldberger and Luis A N Amaral and Leon Glass and Jeffrey M Hausdorff and Plamen Ch. Ivanov and Roger G Mark and Joseph E Mietus and George B Moody and Chung-Kang Peng and H Eugene Stanley},
   doi = {10.1161/01.CIR.101.23.e215},
   issue = {23},
   journal = {Circulation},
   pages = {e215-e220},
   title = {PhysioBank, PhysioToolkit, and PhysioNet  },
   volume = {101},
   url = {https://www.ahajournals.org/doi/abs/10.1161/01.CIR.101.23.e215},
   year = {2000},
}

@article{MITDB,
   author = {George B Moody and Roger G Mark},
   doi = {10.1109/51.932724},
   issue = {3},
   journal = {IEEE Engineering in Medicine and Biology Magazine},
   pages = {45-50},
   title = {The impact of the MIT-BIH Arrhythmia Database},
   volume = {20},
   year = {2001},
}

@article{Leinster2013,
   author = {Tom Leinster},
   doi = {10.4171/DM/415},
   issn = {1431-0635},
   journal = {Doc. Math.},
   pages = {857-905},
   title = {The magnitude of metric spaces},
   volume = {18},
   year = {2013},
}

@misc{PREPRINT,
  title={Maximum diversity, weighting and invariants of time series}, 
  author={Byungchang So},
  year={2025},
  eprint={2509.11146},
  archivePrefix={arXiv},
  primaryClass={stat.ML},
  url={https://arxiv.org/abs/2509.11146}, 
}

@article{HeFong2019,
   author = {Zonglin He and Youyi Fong},
   doi = {https://doi.org/10.1002/sim.8212},
   issue = {20},
   journal = {Statistics in Medicine},
   pages = {3936-3946},
   title = {Maximum diversity weighting for biomarkers with application in HIV-1 vaccine studies},
   volume = {38},
   url = {https://onlinelibrary.wiley.com/doi/abs/10.1002/sim.8212},
   year = {2019}
}

@misc{AnosovaKurlin,
      title={Geometric Data Science}, 
      author={Olga D Anosova and Vitaliy A Kurlin},
      year={2025},
      eprint={2512.05040},
      archivePrefix={arXiv},
      primaryClass={math.MG},
      url={https://arxiv.org/abs/2512.05040}, 
}

@misc{GimperleinGoffengLouca2025,
      title={The magnitude and spectral geometry}, 
      author={Heiko Gimperlein and Magnus Goffeng and Nikoletta Louca},
      year={2025},
      eprint={2201.11363},
      archivePrefix={arXiv},
      primaryClass={math.DG},
      url={https://arxiv.org/abs/2201.11363}, 
}

@book{Willard2004,
   author = {Stephen Willard},
   isbn = {9780486434797},
   publisher = {Dover Publications},
   title = {General topology},
   year = {2004}
}
\end{document}